%% file: mainv2.tex
\documentclass[sigconf]{acmart}

  \numberwithin{equation}{section}

\input{YK-macros.tex}

\input{YK-thm-macros.tex}

\RequirePackage{xcolor}

\newenvironment{blue}{\color{blue}}{}

\usepackage{algorithm,algpseudocode}
\usepackage{wrapfig}
\usepackage{subcaption}

\def\ante{\mathrm{ante}}
\def\maj{\mathrm{maj}}

\AtBeginDocument{%
  \providecommand\BibTeX{{%
    \normalfont B\kern-0.5em{\scshape i\kern-0.25em b}\kern-0.8em\TeX}}}

\copyrightyear{2024}
\acmYear{2024}
\setcopyright{acmlicensed}\acmConference[FAccT '24]{The 2024 ACM Conference on Fairness, Accountability, and Transparency}{June 3--6, 2024}{Rio de Janeiro, Brazil}
\acmBooktitle{The 2024 ACM Conference on Fairness, Accountability, and Transparency (FAccT '24), June 3--6, 2024, Rio de Janeiro, Brazil}
\acmDOI{10.1145/3630106.3658929}
\acmISBN{979-8-4007-0450-5/24/06}

\usepackage{acmart-taps}
\begin{document}
\pdfoutput=1
\title{Algorithmic Fairness in Performative Policy Learning: Escaping the Impossibility of Group Fairness}

\author{Seamus Somerstep}
\email{smrstep@umich.edu}
\affiliation{%
  \institution{University of Michigan}
   \city{Ann Arbor}
   \state{MI}
  \country{USA}
}
\author{Ya'acov Ritov}
\email{yritov@umich.edu}
\affiliation{%
  \institution{University of Michigan}
   \city{Ann Arbor}
  \state{MI}
  \country{USA}
}
\author{Yuekai Sun}
\email{yuekai@umich.edu}
\affiliation{%
  \institution{University of Michigan}
   \city{Ann Arbor}
   \state{MI}
  \country{USA}
}







\renewcommand{\shortauthors}{Somerstep, Ritov, and Sun}

\begin{abstract}
In many prediction problems, the predictive model affects the distribution of the prediction target. This phenomenon is known as performativity and is often caused by the behavior of individuals with vested interests in the outcome of the predictive model. Although performativity is generally problematic because it manifests as distribution shifts, we develop algorithmic fairness practices that leverage performativity to achieve stronger group fairness guarantees in social classification problems (compared to what is achievable in non-performative settings). In particular, we leverage the policymaker's ability to steer the population to remedy inequities in the long term. A crucial benefit of this approach is that it is possible to resolve the incompatibilities between conflicting group fairness definitions.
\end{abstract}





\maketitle

\section{Introduction}
Automated decision-making and support systems that rely on predictive models are often used to make consequential decisions in criminal justice \cite{rudin2013Predictive,casselman2015New}, lending \cite{petrasic2017Algorithms}, and healthcare \cite{deo2015Machine}, but their long-term impacts on the population are poorly understood. Most prior work on algorithmic fairness assumes a \emph{static} population and focuses on \emph{allocative equality} \cite{crawford2017Trouble}. For example, consider the common fairness definition equalized odds \cite{hardt2016Equality}. It requires a predictive model to incur false positives and false negatives at equal rates across demographic groups; \ie\ it requires the model to allocate errors equally between groups. It does not consider long-term impacts of errors across demographic groups: for example, it may be easier for members of an advantaged group to overturn an error (as compared to members of a disadvantaged group), so errors are more consequential for the disadvantaged group. In the long term, this can exacerbate inequalities in the population in ways that simply enforcing equalized odds cannot address.

Motivated by concerns about the long-term impacts of predictive models, there is a line of work that embeds predictive models as \emph{policies} in dynamic models of populations and studies how predictive models \emph{steer} the population \cite{heidari2019Longterm,kannan2019Downstream,damour2020Fairness,yin2023LongTerm}. Following this line of work, we study the enforcement of group fairness in the long term. Our contributions are as follows.
\begin{enumerate}
\item We formulate a new long-term group fairness constraint inspired by algorithmic reform/reparation \cite{Davis-Reparation-2021,Green_2022}, as well as long-term versions of the three traditional group fairness constraints.
\item We show that as long as the policymaker has enough flexibility in the way they remedy historical inequities, it is possible for them to steer populations towards a reformed state while maintaining group fairness. As a consequence, this shows that it is possible, in the long term, to simultaneously satisfy traditionally conflicting group fairness constraints. 
\item We provide a reduction framework for computationally enforcing long-term group fairness. The convergence rates and generalization guarantees of this framework are provided.
\end{enumerate}
As mentioned, a key consequence of our results is that in the long term, it is possible to \emph{simultaneously} remedy historical disparities and satisfy multiple group fairness definitions that are traditionally incompatible.  This is in contrast to previous research on compatible group fairness. The common theme of previous work is that of \emph{avoidance}, \ie{} practitioners should alter policies so that the different notions of fairness are fully satisfied at separate times or relaxed versions of group fairness are satisfied simultaneously. Instead, our focus is on \emph{resolution} of incompatibility of group fairness, \ie{} practitioners should implement policies that eliminate the inequity that leads to the impossibility of enforcing multiple forms of group fairness.
\subsection{Related Work}
We first cover the prior work on resolving group fairness incompatibilities, this is a non-exhaustive list and a thorough survey is given in \citep{Raghavan-IncompatibleSurvey-2023}. The study of incompatibilities in group fairness was initiated simultaneously by the impossibility results in the works \citep{kleinberg2016Inherent, chouldechova2017Fair}. In the years since then, researchers have sought to extend these results and modify fairness practices to partially resolve them. The work \citep{berk2017fairness} covers impossibility results for several group fairness metrics derived from confusion matrices in the context of recidivism. Beyond just group fairness, trade-offs between optimal accuracy, fairness, and resource allocation for different groups have also been explored \citep{Pleiss-FairnessCalibration-2017, Donahue_2020}. The line of work \citep{ReichPossibility2021, madras2018Learning, canetti2019soft} studies algorithmic fairness in ML systems with human elements, although ultimately the aggregate decision making process (human plus machine) is still burdened by impossibility theorems. Another avenue of approach is to relax the sufficiency and separation requirements to compatibility \citep{Pleiss-FairnessCalibration-2017, Celis-FairClass-2019, gultchin2022impossibility, Lohaus-TooRelaxed-2020, padh2021addressing}, although direct trade-offs between fairness notions remain, and often such relaxations do not guarantee a desired level of fairness. Philosophically, our work aligns with the recent idea of \emph{substantive fairness}, which argues that policies should aim to eliminate historical inequities, treating causes of disparity rather than symptoms of disparity \citep{Davis-Reparation-2021, Green_2022}. One of the main contributions of this work is to formalize \emph{substantive fairness} into an algorithmic framework and study the feasibility of erasing historical disparities with ML systems.

Concern amongst the fairness community regarding the long-term impacts of predictive models has grown, originating with a line of work that models fair policies and populations as a dynamical system\citep{heidari2019Longterm,kannan2019Downstream,damour2020Fairness}. The long-term fairness framework that we study is based on the recently developed idea of performative prediction, a line of work that studies model-induced distribution shift \citep{perdomo2020Performative, mendler-dunner2020Stochastic, mendler-dunner2022Anticipating, statefulpp-2022-aistats, izzo2021How}.  Closely related lines of work on the enforcement of fairness in performative settings include \citep{rateike2023designing, zezulka2023performativity,yin2023LongTerm}. The work \citep{rateike2023designing} designs a Markov chain oriented framework of performative fairness, while the work \citep{yin2023LongTerm} develops algorithms that minimize cumulative fairness regret in online settings. These generally consider a \emph{stateful} performative prediction setting and cast the task of steering the underlying dynamical system as an optimal control or reinforcement learning problem. In contrast, our problem setting prioritizes the discrimination in the steady state and is not concerned with intermediate time steps. Finally, the authors of \citep{zezulka2023performativity} propose procedures for estimating discriminatory outcomes of policies. 


Although it predates performative prediction, strategic classification \citep{hardt2016Strategic} is a common example of performative prediction. As such, attention has been paid to the study of fairness in strategic settings. The authors of \cite{Estornell-StratFair-2023} show that traditional fairness interventions in strategic settings can actually exacerbate discrimination. A similar perspective on traditional fairness constraints in strategic settings is given in \citep{liu2018Delayed}.  In \cite{zhang-interventions-2022} the effect of fairness interventions on the incentive for strategic manipulation is studied. The work \cite{ hu2019Disparate} shows that \emph{cost discrimination} in strategic classification will cause violations of group fairness constraints in the long term. The authors of \cite{Liu2023DisparateEquilibria} also study cost discrimination in strategic classification, showing that subsidies for the disadvantaged group can alleviate discrimination concerns. The common goal for each of these works (and others in this area) is to study how strategic interventions make discrimination worse if either no fairness intervention is used or if a *traditional* fairness intervention (that does not account for strategic agents) is used. The key difference between this line of work and our work is that this line of work considers strategic behavior as a problem to be overcome, while we leverage strategic behavior for greater fairness achievement than is possible in non-strategic environments. 

The running example of performative prediction/ a strategic setting in this paper is labor market models; the study of such models has a rich history in economics and long precedes the idea of performative prediction. A comprehensive survey of this field is given in \citep{FANG2011133S}. The works \citep{arrow1971, phelps1972} presented initial labor market models and analyzed the equilibrium of these models at the worker level. In \cite{phelps1972}, discrimination in labor markets due to exogenous groups is studied, while \cite{arrow1971} surprisingly shows that even markets with endogenous groups will have discriminatory equilibrium. Later, the authors \citep{coate1993Will} formulated these ideas in the celebrated \emph{Coate and Loury} model of labor markets. Extensions on this model are numerous; the authors of \cite{moro2004general} and \cite{moro2003Affirmative} develop a model where wages are set by inter-employer competition and groups of workers actually benefit from discrimination of others. The line of work \cite{FryarColorBlindAA, FryarWinnerTakeAll, FryarValuingDiversity, NBERw23811} studies the efficacy of color-blind policies in preventing discrimination. We will primarily work with our own models of labor markets, inspired by \citep{somerstep2024learning, shravit-caus-strat-LR-2020}.

Our primary goal is to study the feasibility of (fairly) equating response distributions between two groups in the long term. This is closely related to the goal of incentivizing agents to improve some desirable quality. The authors of \citep{miller2020Strategic} show that this involves causal modeling, and our work is no exception: various labor market models will play the role of the causal model in our study. The works \cite{milli2019Social, somerstep2024learning} focus on improving the agents' overall welfare; they show that there is often a trade-off between the learner and the agents' utilities. \emph{Performative power} \cite{hardt2022Performative} measures the ability of the learner to steer the agents in performative prediction. As we shall see, the firm in the aforementioned labor market models possess enough performative power to equate ex post response distributions despite ex ante disparities between the two groups. Finally, the causal strategic labor market model that we develop is an example of an \emph{outcome performativity} problem. \emph{outcome performativity} is introduced in \cite{kim2023Making} along with efficient omniprediction algorithms for outcome performativity problems. In general, the novelty of our work is our focus on a) driving improvement with the goal of equating disparate groups and b) doing so without discriminating against the advantaged group.

\section{Group fairness in performative policy learning} 
\label{sec:equality-equity-compatible}
Consider a binary classification problem in which samples correspond to a population of individuals invested in learned policies (often referred to as strategic agents) characterized by $Z = (X, Y) \in \cX \times \{0,1\}$ and a protected attribute denoted $G$. This setting is characterized as a problem in performative prediction. Performative prediction is a distribution shift setting, where the implementation of a policy $f: \cX \times G \rightarrow \{0,1\}$ triggers responses from the individuals invested in the policy, leading to a new distribution of data $\cD(f,G) \in \Delta(Z \times G)$. Throughout, we will assume that group proportions remain constant (say $\Pr(G=g) = \lambda_g$) for any policy $f$. Under this assumption, we can write  
\begin{equation}
    \cD(f,G) = \sum \lambda_g \cD(f,G)|G=g \triangleq  \sum \lambda_g \cD(f,g)
\end{equation}
In performative prediction, The policy maker's goal is to make a policy that maximizes their expected reward, \emph{taking into account the strategic response of the individuals to their decisions}:
\begin{equation}
\max\nolimits_{f\in\cF}\text{EPR}(f)\triangleq \sum \lambda_g\Ex_{Z'\sim\cD(f, g)}\big[r(f(\cdot, g);Z')\big],
\label{eq:performative-policy-learning}
\end{equation}
where $\cF$ is a policy class, $\cD(f, g)$ is the distribution map that encodes the long-term impacts of the policy on the subset of the population with sensitive trait value $G=g$, $r(f(\cdot, g);z)$ is the policy maker's reward function that measures their reward from applying a policy $f$ to an agent $z$ (this agent also has group membership $g$), and $\lambda_g$ is the proportion of the population with sensitive trait value $G=g$. The objective function in \eqref{eq:performative-policy-learning} is often called the performative (expected) utility and measures the \emph{ex post} reward of policies. Throughout this paper, we will mark random variables drawn from an \emph{ex post} distribution as $Z' = (X', Y')$. 

A recurring instance of performative policy learning in this work is the hiring firm's problem in Coate-Loury-type models of labor markets \cite{FANG2011133S}.
\begin{example}[Continuous labor market example \citep{somerstep2023Learning}]
\label{ex: continuous labor market model}
Consider an employer that wants to hire skilled workers who reside in one of two identifiable groups $G \in \{A , D\}; G \sim \text{Ber}[\lambda]$. The workers are represented as $(S, X, Y, G)$ quintuples. $S\in \mathbb{R}$ is a worker's (latent) base skill level, $Y \in \{0,1\}$ a workers productivity and $X\in \mathbb{R}$ be a noisy productivity assessment (\eg\ the outcome of an interview). Throughout, it is assumed that conditioned on $S$, productivity is independent of $G$ and that
\[Y|S \deq  \text{Ber}[\sigma(S)].\]

The productivity assessment $X$ is independent of $G$ given $Y$ and specified by a conditional CDF
\[
I(X\mid y)\triangleq\Pr\{X\le x\mid Y=y\}
\]
that decreases in $y$ (at a fixed $x$). We note that this also specifies the generation of $X$ as
\[
\Phi(X\mid s) \triangleq \sigma(s) I(X|1) + (1-\sigma(s)) I(X|0).
\] 
Intuitively, the assumption that $I(X\mid y)$ decreases in $y$ at a fixed $x$, requires the productivity assessment to be ``unbiased'' (in the sense of a statistical test). Under this unbiased assumption, the optimal hiring policy for the firm is of the form $f(x, \theta, g) = \mathbf{1}_{x \geq \theta_g}$. Let $u(f,y)$ be the firm's received utility from hiring ($f(x)=1$) or not hiring ($f(x)=0$) a worker with productivity $y$. The firm's expected utility for policy $f:\reals \times \{A, D\}\to\{0,1\}$ is $\Ex\big[u(f(X, g),Y)\big]$, so the firm's utility maximization problem is
\begin{equation*}
\max\nolimits_{f} \sum \lambda_g \Ex\big[u(f(X, g),Y)\big].
\end{equation*}
As in \citep{coate1993Will}, we allow the workers to improve their skills (at a cost) in response to the firm's policy. Let $w > 0$ be the wage paid to hired workers and $c_g(s,s') > 0$ be the sensitive trait dependent cost to workers of improving their skills from $s$ to $s'$. We assume $c$ is non-increasing in $s$ and non-decreasing in $s'$. The expected utility a worker with group membership $G=g$ receives from increasing their skill level from $s$ to $s'$ is
\[\textstyle
u_w(f,s,s',g)\triangleq\int_\reals wf(x, g)d\Phi(x\mid s') - c_g(s,s'),
\]
so a strategic worker changes their skill level to maximize their expected utility. We encode the \emph{ex-post} workers' skill level, skill level assessment, and productivity as 
\[
\begin{aligned}
S'\triangleq\argmax_{s'}u_w(f,s,s', g), \\
X'\mid S'\sim \Phi(x\mid S'), \\
Y' \mid S' \sim \text{Ber}[\sigma(S')]
\end{aligned}
\]
Note that the (conditional) distribution of a worker's skill level assessment, given their skill level, remains the same before and after the worker changes their skill level. To account for the strategic behavior of the workers, the firm solves the \emph{performative} policy learning problem:
\begin{equation*}
\max\nolimits_{f} \sum \lambda_g \Ex_{Z' \sim \cD(f,g)}\big[u(f(X', g),Y')\big]
\end{equation*}
\end{example}
The workers do not respond instantly to the employer's hiring policy; it takes them a while. Thus, we interpret $\cD(f, G)$ as the \emph{long term} distribution of the workers' skill levels and assessments in response to the employer's hiring policy. More concretely, imagine a labor market in which the workers slowly turn over: new workers enter the workforce and old workers retire constantly. As workers enter the workforce, they make their human capital investment decisions in response to the employer's (contemporaneous) hiring policy. Over a long period, the labor force population will converge to $\cD(f,G)$.

\subsection{Standard fairness constraints are insufficient in performative prediction}
The standard way to enforce fairness in policy learning problems is to equalize certain fairness metrics between demographic groups (indicated by a demographic attribute $G\in\cG$). This is often done by imposing fairness constraints on the policy learning problem. In general, fairness constraints fall into one of three types:
\[
\begin{aligned}
&\text{(demographic parity DP)}   & f(X)\ind G,\\
&\text{(separation)}  & f(X)\ind G\mid Y, \\
& \text{(sufficiency)} & Y\ind G\mid f(X).
\end{aligned}
\]
To see each of the traditional group fairness constraints in action, consider example \ref{ex: continuous labor market model}. Enforcing separation requires identical \emph{ex-ante} hiring rates between workers from the advantaged and disadvantaged groups with the same skill level, while enforcing sufficiency requires the \emph{ex-ante} (distribution of) skill levels of the hired workers from the majority and minority groups to be the same.
Finally, enforcing demographic parity simply requires that the \emph{ex-ante} hiring rites for individuals be the same across the groups.

In the long-term setting, there are two main issues with such group fairness constraints. First, they only focus on the policy and not its long-term impacts on the population: the policy $f$ appears in all three constraints, but the distribution map $\cD(f, G)$ that encodes the long-term impacts of $f$ is absent. In other words, group fairness constraints enforce equal treatment, but ignore the long-term impacts of equal treatment on the population. Consider example \ref{ex: continuous labor market model}, only \emph{ex-ante} quantities of the workers are considered. This leads us to consider constraints that focus on the long-term impacts of the policy on the population. Instead of enforcing \emph{ex-ante} equal treatment, we seek \emph{ex-post} equality of certain fairness metrics.

Second, traditional group fairness constraints are also plagued by incompatibilities. Despite the intuitive nature of DP, separation and sufficiency, \citep{chouldechova2017Fair,kleinberg2016Inherent} prove that it is generally impossible for a policy to simultaneously satisfy two of DP, separation and sufficiency. 
\begin{theorem}[\citet{chouldechova2017Fair,kleinberg2016Inherent}]
\label{thm: impossibility}
    It is impossible to find a joint distribution on $(f(X), Y, {G})$ that satisfies two of DP, separation, and sufficiency, unless one of the following hold:
    \begin{enumerate}
        \item $\Pr(f(X) = Y) = 1$
        \item $Y \ind \mathcal{G}$
    \end{enumerate}
\end{theorem}
This result is generally interpreted as an impossibility result: it is impossible to simultaneously satisfy separation and sufficiency. Although there are two cases in which separation and sufficiency are compatible, they are considered pathological. The first case, perfect prediction, is generally unachievable because the Bayes error rate in most practical prediction problems is non-zero. The second case, independent responses, is also considered pathological because the policymaker has no control over the distribution of responses. However, in performative settings, the policymaker can steer the population so that response distributions are equal, suggesting that it is possible to resolve the incompatibilities between separation and sufficiency in performative settings by enforcing group-independent responses \emph{ex-post}. In the next section, we build on this observation to resolve the incompatibility between separation, sufficiency, and demographic parity in performative settings.

\subsection{Equality of Outcomes}
The preceding developments suggest that the goal of a fairness-conscious policymaker should be to eliminate disparities between demographic groups in the long term (instead of myopically enforcing group fairness regardless of the long term impacts). Because we are interested in equating the different demographic groups \emph{ex-post} (where an individual from each group ends up) rather than equating the different demographic groups \emph{ex-ante} (where an individual from each group comes from), we refer to this constraint as equality of outcomes.
\begin{definition}[Equality of outcomes]
\label{def:equality of outcomes}
A policy $f$ satisfies equality of outcomes with respect to metric $m(f,g)$ if and only if $m(f,g)$ is constant over each possible value of the sensitive trait $g$.
\end{definition}
We emphasize that fairness metrics $m(\cdot, \cdot)$ are not metrics in the distance sense, but rather a quantity that measures a long-term outcome of interest for strategic individuals. Since we are interested in \emph{ex-post} fairness, each metric will measure quantities associated with the ex-post distribution ($\cD(f, G)$). Finally, for simplicity of presentation, we will present each metric for the case of binary classification $(\ie{} \text{ both }f(X',g) \text{ are in } Y' \in \{0,1\})$. The extension to the multiclass or continuous case is straightforward. 

The goal of algorithmic reform in example \ref{ex: continuous labor market model} is to equalize the disparities in human capital investment among minority and majority workers. Choosing the fairness metric in definition \ref{def:equality of outcomes} to be worker productivity, we can encode this goal as an instance of definition \ref{def:equality of outcomes}.
\begin{definition}[equality of responses]
\label{def:ex-post-equity}
    A policy $f$ satisfies equality of responses if it satisfies equality of outcomes with respect to the metric $m_{\text{res}}(f, g) = \mathbb{E}_{\cD(f, g)}[Y']$. 
\end{definition}
\noindent Equal responses is a concept of fairness unique to the long-term setting (and is our interpretation of the aim of algorithmic reform/reparation); we note that it does not imply the long-term analog of group fairness defined next. We opt to refer to equality of outcomes with respect to any metric containing the policy $f$ as equality of treatment.
\begin{definition}[Equality of treatment]
\label{def:ex-post-group-fairness}

\noindent A policy $f$ satisfies equality of treatment if $f$ meets all the following criteria.
\begin{enumerate} 
    \item  The policy $f$ satisfies equality of outcomes with respect to metric $m_{\text{par}}(f, g) = \mathbb{E}_{\cD(f,g)}[f(X', g)]$. 
    \item The policy $f$ satisfies equality of outcomes with respect to both metrics 
    $$m_{\text{FPR}}(f, g) = \frac{\mathbb{E}_{\cD(f,g)}[\mathbf{1}\{Y'=0\}\mathbf{1}\{f(X', g)=1\}]}{\mathbb{E}_{\cD(f,g)}[1-Y']}$$  $$m_{\text{FNR}}(f, g) = \frac{\mathbb{E}_{\cD(f,g)}[\mathbf{1}\{Y'=1\} \mathbf{1}\{f(X', g)=0\}]}{\mathbb{E}_{\cD(f,g)}[Y']}$$
    \item The policy $f$ satisfies equality of outcomes with respect to both metrics 
    $$m_{\text{PPV}}(f,g) = \frac{\mathbb{E}_{\cD(f,g)}[\mathbf{1}\{Y'=1\}\mathbf{1}\{f(X', g)=1\}]}{\mathbb{E}_{\cD(f,g)}[f(X',g)]}$$
    $$m_{\text{NPV}}(f,g) = \frac{\mathbb{E}_{\cD(f,g)}[\mathbf{1}\{Y'=0 \}\mathbf{1}\{f(X', g)=0\}]}{\mathbb{E}_{\cD(f,g)}[1-f(X',g)]}$$ 
\end{enumerate}
\end{definition}
\noindent We we wish to emphasize that requirements one, two and three are simply the long term analogs of demographic parity, separation and sufficiency respectively. Thus, the enforcement of equality of treatment implies the long-term enforcement of multiple fairness constraints that are incompatible in the short term. 

Of course, our ultimate goal is for policymakers to implement policies that satisfy equality of responses and equality of treatment \emph{simultaneously}. We point out that these goals are not necessarily disjoint. As previous impossibility theorems have shown, in most cases satisfying equality of responses is a prerequisite to satisfying equality of treatment. We show in the following proposition that satisfying equality of treatment and equality of responses is equivalent to enforcing the independence of the joint distribution $(Y', f(X', G))$ and $G$. 
\begin{proposition}
\label{prop:joint-dist-ind}
    A policy $f$ satisfies equality of treatment and equality of responses if and only if the joint distribution $(Y', f(X', G))$ is independent of $G$.
\end{proposition}
As mentioned, equality of responses is a mathematical formalization of the line of work on algorithmic reform/reparation \citep{Davis-Reparation-2021, Green_2022}. This line of work ``escapes'' from incompatibilities between group fairness definitions by questioning the goal of satisfying those definitions. It argues that the underlying goal of enforcing fairness is to remedy injustices. From this perspective, traditional algorithmic fairness definitions are merely a flawed indicator of the true goal, so it is inconsequential if they are incompatible. We encode the goal of reform or reparation mathematically as closing or eliminating disparities in the responses of interest among demographic groups. Furthermore, we study the enforcement of equality of responses and equality of treatment simultaneously. In our formulation, this implies the attainment of algorithmic reform/reparation \emph{without} discrimination. In \citep{Davis-Reparation-2021, Green_2022}, it is often argued that reform can only be achieved by discriminating against the advantaged group. In contrast to this, our analysis will show that it is often possible to achieve reform (equality of responses) fairly (equality of treatment).

\section{Feasibility of equality of outcomes}

A crucial and non-trivial question remains. Is it feasible to enforce equality of treatment and equality of responses simultaneously? Note that this requires the same policy to equate \emph{ex-post} responses and to satisfy equality of outcomes with respect to each long-term group fairness metric. If we deploy one policy to steer the population so that the response distributions are equal and another policy to satisfy the equality of treatment constraints, the second policy may steer the population away from equated responses, and we end up in a cycle that achieves neither goal. Unfortunately, the goal of equal responses and equal treatment may not be possible, even in performative settings; \ie{} there may not be a policy that achieves both the treatment and the response goals. This is because the distribution map $\cD(\theta, G)$ depends on the sensitive attributes (\eg\ because it encodes inequities in the \emph{ex-ante} distribution or inequities in the agent response map). Thus, equating the responses \emph{ex-post} requires disparate treatment of the group (\ie{} the policymaker cannot simply implement the same policy for each group). On the other hand, equality of treatment forbids a disparate allocation of \emph{ex-post} errors between groups. In this section, we study the feasibility of enforcing equality of treatment (and thus equality of responses) in labor market models.

\subsection{Impossibility Results}
We start by establishing impossibility results to elucidate problem structures that preclude equality of treatment and equality of responses in labor market models. We consider two types of disparities: human capital investment cost disparities and \emph{ex-ante} skill disparities. In order to introduce \emph{ex-ante} skill disparities, we define the notion of stochastic dominance.

\begin{definition}[stochastic dominance]
    \label{def:stochastic-dominance}
    Consider two real valued random variables $A$ and $B$. Then $A$ {stochastically dominates} $B$ if for all $x \in \mathbb{R}$,  $\Pr(A \leq x) \leq \Pr(B \leq x)$.
\end{definition}
We return to the continuous labor market model \ref{ex: continuous labor market model} in which an employer hires workers from two demographic groups.  In order to keep the example as equitable as possible (so that it is as easy as possible for the employers to achieve equality of treatment), recall that we assume that the skill assessment process is fair ($X\ind G\mid Y$) and the same wage is paid to hired workers from both groups.
Thus, the only \emph{ex-ante} differences allowed between workers in the two groups are the \emph{ex-ante} distribution of their skill levels and the cost of human capital investment.

\noindent \begin{theorem}
\label{thm: informal}
Assume the following: 
\begin{enumerate}
    \item The worker cost function is of the form $c(s', s, g) = \frac{c_g}{2}(s'-s)^2_+$, and $\text{min}_g(c_g)$ is large enough to enforce strong convexity of the agents optimization problem
    \item Exactly one of two forms of market discrimination is present: 
    \begin{enumerate}
        \item There is a difference in \emph{ex-ante} skill levels (specifically $S|G=A$ {stochastically dominates} $S|G=D$).
        \item The is a difference in the cost of human capital investment of the form $c_A < c_D$.
    \end{enumerate}
\end{enumerate}
Then, under either form of discrimination, group-blind policies that ignore the demographic attribute of the workers $G$ cannot achieve equality of responses. Furthermore, hiring policies that satisfy equality of outcomes with respect to $m_{\text{FPR}}(f,g)$ and $m_{\text{FNR}}(f,g)$ are necessarily group-blind. 
 \end{theorem}
\subsection{Alternative performative models}
We present the preceding negative results to emphasize that simultaneously achieving equality of responses and equality of treatment with respect to various fairness metrics is not trivial. To overcome the difference \emph{ex-ante} between workers in the two groups, employers must offer additional incentives to the \emph{ex-ante} least skilled group to close the skill gap and achieve equal responses. Unfortunately, this prevents them from treating workers from the two groups equally because employers only have a single degree of freedom (the hiring threshold $\theta$). This suggests that it may be possible for employers with more degrees of freedom to equalize \emph{ex-post} response distributions with an (\emph{ex-post}) fair hiring policy.

In this section, we introduce two models, the first is a generic model of causal strategic classification inspired by the work \citep{shravit-caus-strat-LR-2020}. The second is a modification of the causal strategic classification model to better model labor markets. The causal strategic classification model provides a simplified framework for theoretical questions on feasibility and generalization, while also demonstrating that our fairness framework is applicable to a wide variety of learning settings.

\begin{example}[causal strategic classification \cite{shravit-caus-strat-LR-2020}]
    \label{ex: causal strategic classification}
    Consider a learning setting, in which samples correspond to strategic agents that posses features and a sensitive trait $(X, {G}) \in \mathbb{R}^d \times \{A,D\}$. Conditioned on sensitive trait membership ${G}$, features $X$ are generated from the \emph{ex-ante} distribution 
    \[X|G \deq P_G.\] Conditioned on the features $X$, the agent responses $Y \in \{0,1\}$ are generated via the Bernoulli variable
\[Y|X \deq  \text{Ber}[\sigma(\beta^TX)].\] 
Note that this implies $Y \ind G | X$. The learner wishes to accurately classify the agents using the features and sensitive trait. As such, the learner deploys predictions $f(X, G) \in \{0,1\}$ generated with the model
\[f(X, G)| X, G \deq \text{Ber}[\sigma(\theta_G^T X)]. \] 
In response to model choice $\theta_g$, an agent (with trait $g$ and \emph{ex-ante} features $x$) is allowed to take some action $a$ to improve their standing with the learner (at cost $C_g(a)$). The agents act rationally, optimizing their utility:
\[a({\theta}_g) = \argmax_{a' \in \mathbb{R}^k} [{\theta}_g]^T [x + M_{d \times k}a'] - C_g(a')\]
Upon selecting action $a({\theta}_g)$ the agent's features are \emph{ex-post} $ x' = x + Ma({\theta}_g)$. The matrix $M$ is an \emph{effort conversion} matrix, encoding the improvement of the feature $x_i$ from the action $a_j$ for each $i, j \in [d] \times [k]$.  

At the population level, the ex-post feature distribution $X'$ for the group $g$ is given by
\[ P'_g \deq T_{\#}(P_g; \theta_g, M, C_g); \text{ where }T(x; \theta_g, M, C_g) = x + Ma(\theta_g).\]
Conditioned on $G$ and $X$ \emph{ex-post} responses are generated by 
\[Y'|X' \deq  \text{Ber}[\sigma(\beta^TX'),\] 
and ex-post predictions are generated by
\[f(X', G)| X', G \deq \text{Ber}[\sigma(\theta_G^T X')]. \]
To prevent arbitrary inflation of all agent outcomes, the learner is subject to regularization penalty $||\theta||_2^2$. From the above, the learners ex-post risk is
\[\textstyle R(\theta) = \sum_{g\in\mathcal{G}} \lambda_g [\Pr(f(X', G) \neq Y'|G=g)+||\theta_g||_2^2].\]  
\end{example}
To better model a labor market, we modify example \ref{ex: causal strategic classification}. The learning setting will correspond to a labor market and the strategic agents will correspond to workers. The key difference is that, rather than features, workers now have a profile of latent skills, and each skill contributes to the productivity of a worker. Consequently, firms now view a noisy measurement of each skill based on a factor model. This provides workers with multiple ways of investing in their human capital and firms with flexibility in their hiring policies. As we shall see, this additional flexibility is crucial for the enforcement of equality of treatment.
\begin{example}[Modified Labor Market Model]
\label{ex: Modified Labor Market}
Consider the causal strategic classification set up (example \ref{ex: causal strategic classification}). In the context of a labor market, the learner corresponds to a hiring firm, and each strategic agent is a worker. Workers are encoded by pairs $(S, {G}) \in \mathbb{R}^d \times \{A,D\}$, with $S$ corresponding to a \textbf{latent} skill profile; as before $S \not \ind G$ and we say $S|G=g \deq P_g$. The productivity of the workers in the group $g$ is generated by $Y|G=g \sim \text{Ber}[\sigma(\beta^T S)]; {} S \sim P_g$. Rather than observing latent skill profiles, the hiring firm views interview outcomes generated by $X = \Lambda S + \epsilon$, with $\Lambda \in \mathbb{R}^{p \times d}$ a matrix of factor loadings.
This skill assessment model is motivated by item response theory (IRT) models of test outcomes \cite{lord1980Applications}. The results of the interviews are used to make hiring decisions through policy $f(X, G)|X,G \sim \text{Ber}[\sigma(\theta_G^T X)].$

A worker in group $g$ with an initial skill profile $s$ can take action (possibly training, studying, or additional education) in response to $\theta_g$ through the causal strategic classification response mechanism
\[a({\theta}_g) = \argmax_{a' \in \mathbb{R}^k} [\Lambda^T{\theta}_g]^T [s + M_{d \times k}a'] - C(a', g).\]
The ex-post skill profiles are $S' = S + Ma(\theta_g), S \sim P_g$, which propagates to the ex-post interviews $X'$, productivity $Y'$ and hiring decisions $f(X', g)$. The firm seeks to maximize some (regularized) ex-post reward/profit
\[R(\theta) = \sum_{g\in G} \lambda_g \mathbb{E}[r(S', X', \theta_g)|G=g] - ||\theta_g||^2\]
\end{example}
\subsection{Feasibility of Equality of Treatment in Alternative Models}
In contrast to the models in example \ref{ex: continuous labor market model}, in the causal strategic classification setting (and the alternative labor market mode)l, the set of policies that enforces equality of treatments (and thus equality of responses by \ref{prop:joint-dist-ind}) is non-empty. In fact, it contains a stratified manifold of dimension $O(d)$, so the set is quite large in some sense. We will study this set in two scenarios: correcting \emph{ex-ante} feature/skill disparities and correcting cost-of-improvement disparities. For the purposes of a theoretical analysis, we operate under some simplifying assumptions.
\begin{assumption}
\label{ass: labormarket-simplify}
\phantom{newline}
    \begin{enumerate}
        \item The agent (worker) cost is quadratic: $C(a,g) = \frac{c_g}{2} ||a||_2^2$, the effort matrix $M$ is of the form $M = \text{diag}[{B}]; {B} \in \{0,1\}^d$.
        \item The ex ante features (latent skill profiles) and interview outcome distributions are Gaussian. In Example \ref{ex: causal strategic classification} $X | G  \deq \mathcal{N}(\mu_G, I)$ while in example \ref{ex: Modified Labor Market} $S | G  \deq \mathcal{N}(\mu_G, I)$ and $\epsilon \deq \mathcal{N}(0, I)$.
    \end{enumerate}
\end{assumption}
The quadratic cost assumption is standard in the strategic learning literature, \citep{shravit-caus-strat-LR-2020, izzo2021How, jagadeesan2023SupplySide}. The effort matrix is of the form $\text{diag}[B]; B \in \{0,1\}^d$, if each skill is improved by a distinct action and only \emph{some} skills can be improved. Each assumption (including normality) is primarily for mathematical convenience; we expect that feasibility will hold under a wide class of choices for $C(a), \Lambda, M$ and measures on $S, \epsilon$.

Under assumption \ref{ass: labormarket-simplify}, the feasibility of equality of treatment can be studied by analyzing two disjoint constraint sets, one that pertains to parameters that correspond to ``manipulable" features and another that pertains to ``nonmanipulable" features, which we now define. 
\begin{definition}[Manipulable features]
    Given an effort matrix $M = \text{diag}[B]; B \in \{0,1\}^d$, and a general feature vector $v \in \reals^{p}$, we let $v_m = \{v_i \in v; \mathbf{1}_{\{M_i = 1\}}\} $ and $v_u = \{v_i \in v; \mathbf{1}_{\{M_i = 0\}}\} $ be the manipulable and nonmanipulable features, respectively. 
\end{definition}

\noindent Additionally, we will assign $2d_m$, $2d_u$ as the dimensions of the parameter spaces $(\theta_{m, A}, \theta_{m, D}), (\theta_{u, A}, \theta_{u, D})$. 

\begin{theorem}
\label{thm: ex-ante discrimination}
   Consider the learning setting of example \ref{ex: causal strategic classification} with the minor assumption that the vectors $\{(\mu_{A,m}, -\mu_{D,m}), (\beta_m, -\beta_m)\}$ are not co-linear, and the vectors $\{(\mu_{A,u}, -\mu_{D,u}), (\beta_m, -\beta_m)\}$ are not co-linear. Suppose that one of the two following forms of discrimination is present:
   \begin{enumerate}
       \item Ex-ante distribution discrimination: $c_A = c_D = 1$, but $\mu_A^T \beta > \mu_D^T \beta $, 
       \item  Cost of improvement discrimination: $\mu_A = \mu_D = 0$ but $c_A < c_D$.
   \end{enumerate}
Then under (1) or (2) there exist stratified manifolds $\cM_m$, $\cM_u$ such that $\text{dim}[\cM_u] = d_u-2$, and $\text{dim}[\cM_m] = d_m-2$  and any learner decision $\theta = (\theta_A, \theta_D)$ that satisfies $(\theta_{A,m}, \theta_{D,m}) \in \cM_m$ and $(\theta_{A,u}, \theta_{D,u}) \in \cM_u$ also satisfies equality of treatment and equality of responses.
\end{theorem}

\begin{corollary}
    \label{corr:ex-ante discrim Lambda}
    Consider the modified labor market model (example \ref{ex: Modified Labor Market}). Assume that worker discrimination of the form (1) or (2) is present.  Then if $d_m = d_u \approx d/2$ there exists a stratified manifold  $\cM$ of dimension $\cO(d)$
    such that any $\theta = (\theta_A, \theta_D) \in \cM$ satisfies equality of treatment.
\end{corollary}
Theorem \ref{thm: ex-ante discrimination} and corollary \ref{corr:ex-ante discrim Lambda} state that in the causal strategic classification model/modified labor market model, the set of policies that enforce long-term fairness and correct differences in worker skills or costs contains a manifold of dimension $\cO(d)$. Here, $d$ is interpreted as the ``number of skills" a worker can possess. This result clarifies the importance of flexibility in human capital investment. If $d$ is not large enough, then the feasible subsets of Theorem \ref{thm: ex-ante discrimination} will be either too small to provide policy makers flexibility or even empty in extreme cases. The assumptions of theorem \ref{thm: ex-ante discrimination} also provide an important form of \emph{policy maker} flexibility; the existence of skills that are immune to performative effects implies the existence of policy parameters that can be adjusted to change error rates between groups \emph{without} impacting downstream responses.

Open questions on the feasibility of equality of treatment in labor markets with multiple forms of discrimination or continuous outcomes remain. The assumption that only one form of discrimination is present (cost of education discrimination \emph{or} \emph{ex-ante} skill discrimination) is necessary for our analysis but not necessary for feasibility (see Figure \ref{fig:experiment} for a numerical example). The discrete nature of worker productivity and firm hiring decisions is also not strictly necessary; for example. The argument of theorem \ref{thm: ex-ante discrimination} immediately implies feasibility (assuming Gaussian features) in the Causal Strategic Least Squares model posited in \citep{shravit-caus-strat-LR-2020}. This model is both the inspiration for our labor market model and can also be interpreted in the context of a labor market with continuous worker productivity.
\section{A reduction algorithm for equality of outcomes}
In practice, a policy maker often does not have complete knowledge of the \emph{ex-ante} and \emph{ex-post} distributions but instead only observes some samples from the \emph{ex-ante} distribution and has some model for how individuals respond to their policy. We turn to the question of implementing equality of outcomes under such conditions, providing a reduction algorithm (inspired by \citet{agarwal2018Reductions}) adapted to the performative setting. By proposition \ref{prop:joint-dist-ind} a policy maker can implement equality of treatment and equality of responses by enforcing the independence of the joint distribution $(Y', f(X',G))$ and $G$. We propose that the policymaker enforce this through a series of moment inequality constraints:
\[M\mu(f) \leq c,\]
\[\mu(f)_{ij} = \mathbb{E}[h_i(f(X'), Y')|G = g_j].\]
Throughout this section, we will work with the causal strategic classification setting (example \ref{ex: causal strategic classification}) in the two group setting. In this model, equal responses and equal treatment can be enforced through a series of 6 moment constraints.
\begin{example}
\label{ex: moment-binary-class}
    Consider the causal strategic classification example \ref{ex: causal strategic classification} with two possible groups $ \{A, D\}$. In this setting, the condition $(Y', f(X',G)) \ind G$ is equivalent to the constraint:
    \[\begin{pmatrix}
    1 & -1 & 0 & 0 & 0 & 0\\
    -1 & 1 & 0 & 0 & 0 & 0\\
    0 & 0 & 1 & -1 & 0 & 0\\
    0 & 0 & -1 & 1 & 0 & 0\\
    0 & 0 & 0 & 0 & 1 & -1\\
    0 & 0 & 0 & 0 & -1 & 1\\
\end{pmatrix}\\
\begin{pmatrix}
    \mathbb{E}[Y'|A] \\
    \mathbb{E}[Y'|D] \\
    \mathbb{E}[f(X',G)|A] \\
    \mathbb{E}[f(X',G)|D] \\
    \mathbb{E}[f(X',G)Y'|A] \\
    \mathbb{E}[f(X', G)Y'|D] \\
\end{pmatrix} \leq 0_6,\]
while the \emph{ex-post} risk is given by
\[\textstyle\text{EPR}(\theta) = \sum_{g\in\{A, D\}} \lambda_g [\Pr_{\cD(f,g)}(f(X',g) \neq Y')+||\theta_g||_2^2].\]
\end{example}
The case of multi-class classification or multiple sensitive traits is relatively similar. For multiclass classification, enforcement of the independence requirement will require the addition of higher-order moment constraints, and moment constraints may simply be repeated for each possible sensitive trait combination in the case of multiple sensitive traits.

Recall that the policy maker is vested in minimizing some \emph{ex-post} risk, therefore, in aggregate, the policymaker must solve a constrained optimization problem in $f$. The primal problem is
\begin{equation}\textstyle
\label{eq: ex post lagrangian}
\cL(f; \lambda) = \min_{f \in \cF} \max_{\lambda \geq 0} \text{EPR}(f) + \lambda^T(M\mu(f) - c).
\end{equation}
In practice, the policymaker only observes samples $\{Z_i\}_{i=1}^n$ from the \emph{ex-ante} distribution. We will assume that the policymaker has access to some correctly specified model of $\mathcal{D}(f,g)$ at the sample level. For example, a hiring firm in a labor market would need to be aware of each worker's cost-adjusted utility optimization problem. Using \emph{ex-ante} samples $\{Z_i\}_{i=1}^n$ and knowledge of $\cD(f)$, the policymaker can obtain the natural empirical estimates of $\text{EPR}(f)$ and $\mu(f)$ (denoted $\hat{\text{EPR}}(f)$ and $\hat{\mu}(f)$). As an example of obtaining estimates for $\text{EPR}(f)$ and $\mu(f)$ from $\cD(f)$ we return to the modified labor market model.
\begin{example}[Equality of outcomes estimates in causal strategic classification]
\label{ex: estimators}
Recall the setting of Example \ref{ex: causal strategic classification}. Consider an agent with sensitive trait $G=g$ and ex-ante features $x$, upon viewing policy $\theta_g$, the agent invests in their own features via $x' = x +Ma(\theta_g, g).$
Given an \emph{ex-ante} sample of skill features from group $g$, $\{x_{i}\}_{i=1}^n$ a reasonable choice of estimates for $\text{EPR}(f)$ and $\mu(f)$ are
\[\textstyle\widehat{\mathbb{E}[Y'|G=g]} = \frac{1}{n}\sum_{i=1}^n \sigma(\beta^T(x_{i}+ Ma(\theta_g,g))),\]
\[\textstyle\widehat{\mathbb{E}[f(X', G)|G=g]} = \frac{1}{n}\sum_{i=1}^n \sigma(\theta_g^T(x_{i}+ Ma(\theta_g,g))),\]
\[\textstyle\widehat{\mathbb{E}[f(X', G)Y'|G=g]} = \frac{1}{n}\sum_{i=1}^n \sigma(\beta^T(x_{i}+ Ma(\theta_g,g)))\sigma(\theta_g^T(x_{i}+ Ma(\theta_g,g))),\]
\[\textstyle\widehat{\text{EPR}}(\theta_g) = \frac{1}{n}\sum_{i=1}^n \sigma(\theta_g^T(x_{i}+ M_ga(\theta_g,g)))(1-\sigma(\beta^T(x_{i}+ M_ga(\theta_g,g))))\] \[+ (1-\sigma(\theta_g^T(x_{i}+M_ga(\theta_g,g))))(\sigma(\beta^T(x_{i}+ M_ga(\theta_g,g)))).\]
\end{example}

\noindent After attaining estimates $\hat{\mu}(f)$ and $\widehat{\text{EPR}}(f)$, equation \ref{eq: ex post lagrangian} can be replaced with said estimates. Furthermore, for convergence reasons, a $L_1$ norm constraint is placed on the dual variable $\lambda$. Finally, due to statistical error, a relaxation $\hat{c} = c + \nu$ is allowed on the moment constraint. If the learner instead opts to solve the dual problem (the justification for this is expanded upon in Appendix A), the final result is
\begin{equation}\textstyle
    \label{eq: sample dual}
    \cL(f; \lambda) =  \max_{\lambda \geq 0; ||\lambda||_1 \leq B} \min_{f \in \cF} \widehat{\text{EPR}}(f) + \lambda^T(M\hat{\mu}(f) - \hat{c}).
\end{equation}
From here, the iterates of the dual variables are obtained using mirror ascent on the dual variable with the potential function $\phi(\lambda) = -\lambda ln(\lambda)$, which algorithm 1 lays out explicitly.
\begin{algorithm}[H]
\caption{Reduction for Equality of Outcomes}
\label{alg: ex-post-reduction}
\begin{algorithmic}[1]
\Require Error tolerance $\epsilon$, step size $\eta$, samples $\{Z_i\}_{i=1}^n$, fairness constraints $M,\hat{c}, \mu$, initial iterate $v_0$
\For{$t=0, 1, \ldots$}
\State Scale dual: $\lambda_{k,t} \gets B\frac{e^{v_{k, t}}}{1+\sum_k e^{v_{k,t}}} $
\State Get policymakers best decision: $f_t \gets \argmin_{f\in \cF} \widehat{\text{EPR}}
(f) + \lambda_t^T(M\hat{\mu}(f) - \hat{c})$ 
\If{$(\lambda_t, f_t)$ is an $\epsilon$-saddle point}
\State \Return $({\lambda}_t, f_t)$
\Else
\State  Update iterates: $v_{t+1} \gets v_t + \eta_t(M\hat{\mu}(f_t) - \hat{c})$
\EndIf
\EndFor
\end{algorithmic}
\end{algorithm}

Two elements of the algorithm \ref{alg: ex-post-reduction}, are nontrivial: (i) the $\epsilon$-saddle point stopping criteria; (ii) the attainment of the best decision of the policy makers $f_t$.

\textbf{The $\epsilon$-approximate saddle point stopping criteria:}
A primal dual pair $(f_t, \lambda_t)$ is a $\epsilon$-saddle point if the following hold: 
\[\textstyle\mathcal{L}(f_t, {\lambda}_t) \leq \min_{f\in \cF }\cL(f, {\lambda}_t) + \epsilon\]
\[\textstyle\mathcal{L}(f_t, {\lambda}_t) \geq \max_{\lambda \geq 0; ||\lambda||_1 \leq B} \mathcal{L}(f_t, \lambda) - \epsilon\]
Checking the first criteria reduces to a problem in attainment of the policy makers best decision $f_t$\citep{agarwal2018Reductions}. The second requirement requires solving a linear program with an L1 inequality constraint, a well-studied problem \cite{boyd2004Convex}.

\textbf{Obtaining of the policymakers best decision $f_t$:} Attaining the best long-term policy for risk function $\widehat{\text{EPR}}
(f) + \lambda_t^T(M\hat{\mu}(f) - \hat{c})$ is generally a nontrivial problem. Previous works have established methods for obtaining the best policy $f$ under the assumption that the policy maker knows the map $\mathcal{D}(f)$ \citep{a-tale-of-two-shifts, levanon2021Strategic, somerstep2024learning}. Such methods are generally specific to a particular $\cD$, and since our focus is on fairness, we will assume that the policymaker has access to some oracle which produces such an $f$. This is a strong assumption, and our methodology is limited to performative maps that allow for such an oracle.

\subsection{Algorithm 1 in the Modified Labor Market Model}
As an application of Algorithm 1, we study the problem of enforcing equality of outcomes and equality of responses in the modified labor market (example \ref{ex: Modified Labor Market}) when both \emph{ex-ante} distributions \textbf{and} cost of education are different between each group. We assume that $\Lambda = I$, so that the firm has an unbiased estimate of each skill. Figure \ref{fig:experiment} demonstrates the performance of algorithm 1 on a held-out test set of workers. Prior work on long-term fairness studied the impact of enforcing standard fairness constraints in the long term, and, as such, we utilize this as a base line. Algorithm 1 is compared to policies that equalize one \emph{ex-ante} fairness metric, including false positive rate, false negative rate, demographic parity, and sufficiency.
\begin{figure}[H]
  \caption{Sufficiency + separation simultaneously}
  \begin{subfigure}[b]{0.24\textwidth}
    \includegraphics[width=\textwidth]{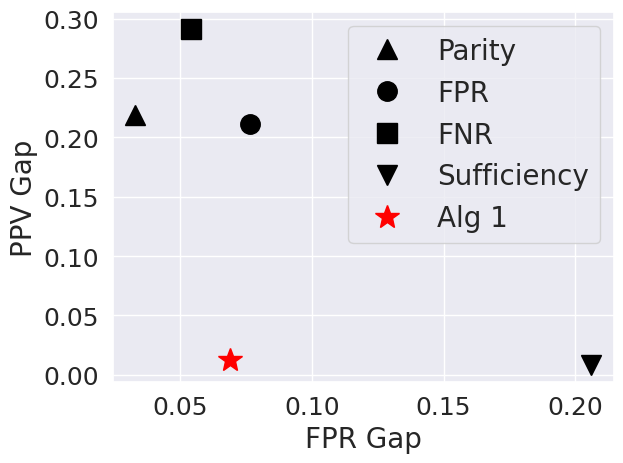}
    \caption{}
  \end{subfigure}
  \hfill
  \begin{subfigure}[b]{0.24\textwidth}
    \includegraphics[width=\textwidth]{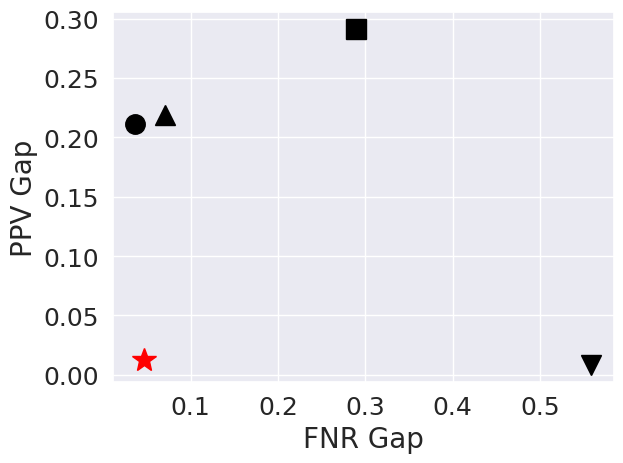}
    \caption{}
  \end{subfigure}
  \hfill
  \begin{subfigure}[b]{0.24\textwidth}
    \includegraphics[width=\textwidth]{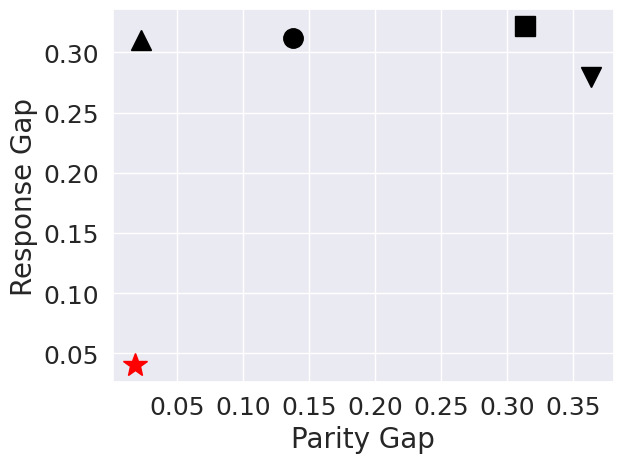}
    \caption{}
  \end{subfigure}
  \hfill
  \begin{subfigure}[b]{0.24\textwidth}
    \includegraphics[width=\textwidth]{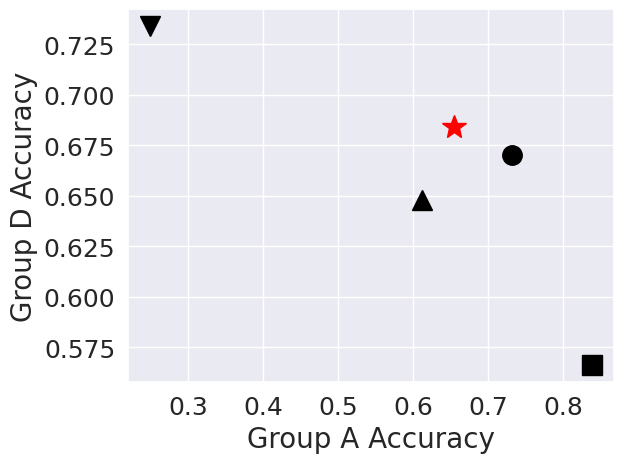}
    \caption{}
  \end{subfigure}
  \label{fig:experiment}
\end{figure}
Figure \ref{fig:experiment} \textbf{(a)} and \textbf{(b)} demonstrate that in the long term sufficiency and separation are compatible, while figure \ref{fig:experiment} \textbf{(c)} shows that in the modified labor market model, a policymaker can satisfy equality of responses with a \emph{fair} policy. Figure \ref{fig:experiment} \textbf{(d)} also demonstrates that enforcing equality of treatment and equality of responses will equate the error rate of a policy between the two groups.
Due to statistical error (exacerbated by the opaque nature of worker skills), perfect fairness is not achieved on the test set. We emphasize that this is not due to an issue of \emph{feasibility}, figure \ref{fig:feasibility exp} demonstrates that algorithm 1 can attain nearly zero fairness violation on the training set.
\section{Summary and discussion}
{In this paper, we studied fairness in performative settings in which the policymaker has the ability to steer the population. We showed that it is possible for the policymaker to remedy existing inequities in the population. In particular, we showed that by equating the distribution of responses $Y$ between groups in classification problems, it is possible for the policymaker to simultaneously satisfy multiple notions of group fairness that are generally incompatible in non-performative settings. However, we also showed that this is not always possible: if the policymaker does not have enough flexibility in how they can equalize base rates, then it is unfortunately impossible, even in performative settings, to resolve the longstanding incompatibilities between group fairness definitions. Another limitation of our approach is that the policymaker must be aware of the long-term impacts of their policies on the population. Although this requirement is necessary, it limits the applicability of the approach. One possible direction for future work is to develop methods that help the policymaker estimate the effects of their policies on the population. Such methods can be combined with our approach to steer sociotechnical systems towards more equitable states.

Our work is also aligned with the goals of algorithmic reform/reparations. By considering reform/reparations mathematically, we show that it is somewhat possible to achieve the goals of reform without unequal treatment. This is especially desirable in application domains in which unequal treatment is illegal or impractical. For example, consider the problem of underrepresentation of women in the tech sector, especially in technical roles \citep{dastin2018Amazon}. The authors of \citet{Davis-Reparation-2021} suggest that employers in the tech sector should adopt a ``reparative'' approach to equalize the representation of men and women, even if this entails explicitly discriminating against men. They justify explicit discrimination by appealing to the historical injustices that led to the dearth of women in the tech sector and the need to remedy such injustices. Although a discriminatory approach is likely to be limited by various labor laws, our results suggest that it may be possible to equalize the representation of men and women while treating men and women fairly. We hope that our results lead to more serious consideration of ``reparative'' approaches in algorithmic decision making.}
\begin{acks}
This paper is based upon work supported by the National Science Foundation (NSF) under grants no. 2027737 and 2113373.
\end{acks}

\bibliographystyle{abbrvnat}
\bibliography{YK,seamus}

\newpage
\onecolumn
\appendix

\section{Theoretical Properties of Algorithm \ref{alg: ex-post-reduction}}
In our development of the ex-post reduction algorithm (\ref{alg: ex-post-reduction}), we opted to solve the Dual problem rather than the primal problem, to achieve a fair policy. This is only justifiable if strong duality of problem \ref{eq: ex post lagrangian} holds. Unfortunately, moment constraints of the form \ref{ex: moment-binary-class} are often not convex and thus strong duality may not hold. To fix this issue (as is often done in other reduction methods), we can slightly generalize algorithm \ref{alg: ex-post-reduction} to allow for randomized policies $Q \in \Delta(\mathcal{F})$, that first select a policy $f$ at random (with $\Pr(Q = f) = Q(f)$), then make a prediction.

As long as group membership is independent of the policy selected, \ie{} for any $Q \in \Delta(\cF)$ the events $1\{G = g\}$ and $1\{Q=f\}$ are independent, one can show that $\mu(Q)$ and $\text{EPR}$ are linear in $Q$. 
\begin{proposition}
\label{prop: linearity}
    Suppose that all $Q \in \Delta(\cF)$ satisfy $Q \ind G$. Then the following holds for all $Q \in \Delta(\cF)$: 
    \begin{enumerate}
        \item $\mu(Q) = \sum_{f \in \cF} Q(f) \mu(f).$
        \item $\text{EPR}(Q) = \sum_{f \in \cF} Q(f)\text{EPR}(f)$
    \end{enumerate}
\end{proposition}
The implication of proposition \ref{prop: linearity} is that both the \emph{ex-post} risk and fairness constraints are linear in $Q$; this in turn will give us strong duality.
In terms of $Q$ the policymaker's primal problem is
\begin{equation}\textstyle
\label{eq: ex post random lagrangian}
\cL(Q; \lambda) = \min_{Q \in \Delta} \max_{\lambda \geq 0} \text{EPR}(Q) + \lambda^T(M\mu(Q) - c).
\end{equation}
Because this problem is linear in $Q$ and $\lambda$, the domains of $Q$ and $\lambda$ are convex, and the equality of treatments constraint is feasible (theorem \ref{thm: ex-ante discrimination}) the solution to \ref{eq: ex post random lagrangian} will be the unique saddle point $(Q^*, \lambda^*)$, which algorithm \ref{alg: ex-post-reduction} (appropriately modified to include randomization) will converge to. Specifically, if $\{f_t\}_{t=1}^T$ and ${\lambda_t}_{t=1}^T$ are the iterates of algorithm \ref{alg: ex-post-reduction}, then the empirical measure $Q_T = \frac{1}{T} \sum_{t=1}^T f_t$ and the mean $\bar{\lambda}_T = \frac{1}{T}\sum_{t=1}^T\lambda_t$ will eventually converge to an appropriate saddle point.
\begin{proposition}[\citet{agarwal2018Reductions}Theorem 1]
Let $Q_T = \frac{1}{T} \sum_{t=1}^T f_t$, $\bar{\lambda}_T = \frac{1}{T}\sum_{t=1}^T\lambda_t$ be the empirical distribution (resp. average) of the primal (resp. dual) iterates of algorithm \ref{alg: ex-post-reduction}. Let $\rho=\sup_f ||M\mu(f)-c||_{\infty}$, $K$ be the total number of moment constraints, and $\eta_t = \sqrt{log(K)+1}/\rho \sqrt{t}$. Then $(\bar{\lambda}_T, Q_T)$ is an $\epsilon_T$ saddle point with
$\epsilon_T = 2\rho B\sqrt{(log(K)+1)/T}$
\end{proposition}
Besides loss of precision due to optimization error, questions on the statistical error of policies produced by algorithm \ref{alg: ex-post-reduction} remain. Unlike the case of optimization error, the statistical error analysis is not identical to the analysis in \citep{agarwal2018Reductions}. This is due to the presence of performativity, which can affect the uniform convergence of any estimator. For concreteness, we consider the causal strategic classification example, recovering the classic parametric rate.
\noindent\begin{assumption}
\phantom{text}
\noindent\begin{enumerate}
        \item The learning setting is example \ref{ex: causal strategic classification} with $M{\mu}(\theta)-{c}$ of the form  in example \ref{ex: moment-binary-class}, and $\widehat{\text{EPR}}(\theta)$, $\hat{\mu}(\theta)$ of the form in example \ref{ex: estimators}.
        \item The parameter space $\Theta \subset \mathbb{R}^d$ is compact and $||X||_{\infty}$ is bounded above with probability one. Furthermore, the response $Ma(\theta)$ is bounded.
    \end{enumerate}
\end{assumption}
\begin{theorem}
\label{thm: error analysis}
    Let $n_A$, $n_D$ denote the number of samples observed by the policymaker from each group. Suppose $(\hat{Q}, \hat{\lambda})$ is an $\epsilon$-saddle point of \ref{eq: sample dual} with $\nu = \cO(\text{min}(n_A, n_D)^{-1/2})$ and $c = 0_6$. Let $Q^*$ minimize $\text{EPR}(Q)$ subject to $M\mu(Q)\leq c$. Then with probability at least $1-7\delta$ the distribution $\hat{Q}$ satisfies
    \[\text{EPR}(\hat{Q}) \leq \text{EPR}({Q^*}) + 2\epsilon + \tilde{\cO}(\text{min}(n_A, n_D)^{-1/2})\]
    \[||M{\mu}(\hat{Q})||_{\infty} \leq \frac{1+2\epsilon}{B} + \tilde{\cO}(\text{min}(n_A, n_D)^{-1/2})\]
    where $\tilde{\cO}$ hides square root dependence on $ln(1/\delta)$. 
\end{theorem}
In practice, randomization is often an unnecessary complication, and for simplicity, experiments are performed utilizing the non-randomized version of algorithm \ref{alg: ex-post-reduction}.
\section{Coate and Loury Result}
In this section we cover a similar impossibility theorem to \ref{thm: informal} for the Coate and Loury model.
\begin{example}[Coate-Loury labor market model \citep{coate1993Will}]
\label{ex: coate-loury}
Consider an employer that wishes to hire skilled workers which reside in one of two identifiable groups $G \in \{A , D\}; G \sim \text{Ber}[\lambda]$. The workers are represented as $(X, Y, G)$ tuples. Here $Y|G=g \in \{0,1\}$ drawn from a $\text{Bernoulli}(\pi_g)$ distribution represents the qualification/skill of a worker, and $X \in [0,1]$ some noisy signal (possible the outcome of an skill assessment) drawn from  CDF $\Phi(X | Y). $ We will assume that $\Phi(X|Y =1)$ stochastically dominates $\Phi(X|Y=0)$, and additionally that $X \ind G | Y$. As such the hiring firm opts to deploy hiring policy $f(x, \theta, g) = 1_{x \geq \theta_g}$. The performative aspect of the model is that post deployment of any hiring policy $\theta_g$, the workers select their qualification level in response to the employer's policy. If $w > 0$ is the wage paid to a worker and $C_g$ is the (random and drawn from CDF $E_g$) cost of attaining qualification, then for a worker from group $g$ with observed cost $c_g$ the utility of each selecting each option of skill $y$ is
\[\textstyle
u_w(\theta_g,y,c_g)\triangleq\begin{cases}\textstyle\int_{[\theta_g,1]} w d\Phi(x\mid 1) - c_g & \text{worker selects $y=1$},\\\textstyle\int_{[\theta_g,1]} w d\Phi(x\mid 0) & \text{if the worker remains unskilled},\end{cases}
\]
Each worker acts rationally and selects the $y$ that maximizes their utility, at the sample level the performative map for a worker with sensitive trait $g$ is {
\[
\begin{aligned}
Y'\gets\argmax_{y'\in\{0,1\}}u_w(f,Y,y'), \\
X'\mid Y'\sim \Phi(x\mid Y').
\end{aligned}
\]
In aggregate, the proportion of qualified workers (in a given group) is updated via 
\[\pi_g(\theta_g) = E_g(w(P(X>\theta_g| Y=1) - P(X>\theta_g|Y=0))).\]
The workers' do not respond instantly to the employer's hiring policy; it takes them a while. Thus we interpret $\cD(f)$ as the \emph{long term} distribution of the workers' skill levels and assessments in response to the employer's hiring policy. More concretely, imagine a labor market in which the workers slowly turn over: new workers enter the workforce and old workers retire constantly. As workers enter the workforce, they make their human capital investment decisions in response to the employer's (contemporaneous) hiring policy. Over a long period, the labor force will converge to $\cD(f)$. To account for the long-term effects of their hiring policy on the labor market, the employer solves the \emph{performative} policy learning problem:}
\[\textstyle R(\theta) = \sum_{g \in G} \lambda_g [ p_+\mathbb{P}(X>\theta_g \mid y=1) \pi_g(\theta_g) - p_-\mathbb{P}(X>\theta_g \mid y=0) (1-\pi_g(\theta_g))],\]
where $p_+$ and $p_-$ are the firm's utilities for hiring a qualified and unqualified worker respectively.
\end{example}
\begin{proposition}
Consider the discrete labor market example \ref{ex: coate-loury}, recall the mechanism of discrimination:
\begin{enumerate}
    \item Wages are independent of group membership.
    \item There is no differential item functioning (DIF). \cite{penfield2006Differential} in the skill assessment ($X\ind G\mid Y$)
    \item Cost of education is discriminatory, \ie{} $E_A(c) \not \deq E_D(c)$.
\end{enumerate}
Thus, the \emph{only} source of discrimination is the cost of education. Suppose that this discrimination is of the following form: the cost random variables $C_A$, $C_D$ are unbounded and $C_D$ stochastically dominates $C_A$, so that group $D$ is strictly disadvantaged through the cost of education. Then, for any hiring policy $\theta = (\theta_A, \theta_D)$ that satisfies equality of treatments (and thus equality of outcomes), it holds that $\pi_A(\theta_A) = \pi_D(\theta_D) = 0$.
\end{proposition}
\begin{proof}
Let $\text{TPR}$, $\text{FPR}$ denote the true positive and false positive rate of a classifier. Note that $\pi_g(\theta_g) = G_g(w(\text{TPR}(\theta_g)- \text{FPR}(\theta_g)))$. Note that ex-post separation would require that $\text{TPR}(\theta_A) = \text{TPR}(\theta_D)$ and $\text{FPR}(\theta_A) = \text{FPR}(\theta_D)$. However by the assumption that $c>0 \implies G_A(c) > G_D(c)$, any policy $(\theta_A, \theta_D)$ that satisfies ex post separation and $w(\text{TPR}(\theta_A)- \text{FPR}(\theta_A))> 0$ satisfies $G_A(w(\text{TPR}(\theta_A)- \text{FPR}(\theta_A))) > G_D(w(\text{TPR}(\theta_D)- \text{FPR}(\theta_D)))$ thus any policy that satisfies ex post equality must satisfy $\pi(\theta_A)=\pi(\theta_D) = 0$.
\end{proof}
\section{Section 2 and Section 3 Proofs}
\subsection{section 2 proofs}
\begin{proposition}
    A policy $f$ satisfies equality of treatment and equality of responses if and only if the joint distribution $(Y', f(X', G))$ is independent of $G$.
\end{proposition}
\begin{proof}[Proof of proposition \ref{prop:joint-dist-ind}:]
\phantom{proof}
Suppose a policy $f$ satisfies fairness definitions \ref{def:ex-post-equity} and \ref{def:ex-post-group-fairness}. We have
\[\phi(f(X'), Y'|G) = p(f(X')|Y', G)p(Y'|G) \stackrel{\text{defs }\ref{def:ex-post-equity}+\ref{def:ex-post-group-fairness}}{(f;z)}{=} p(f(X')|Y')p(Y') = \phi(f(X'), Y') \]
On the other hand suppose $\phi(f(X'), Y'|G=g) = \phi(f(X'), Y')$ for $g \in |G|$. Trivially we must have $f(X') \ind G$ and $Y' \ind G$. For separation note that $p(f(X')|Y', G) = \phi(f(X'), Y'|G)/(p(Y'|G)) = \phi(f(X'), Y')/(p(Y')) = p(f(X')|Y')$. Sufficiency will follow from an essentially identical argument.
\end{proof}
\subsection{Proof of Theorem \ref{thm: ex-ante discrimination}}
\begin{lemma}
\label{lem: strat-resp closed form}
    Under assumption \ref{ass: labormarket-simplify} an agent in group $g$ selects action $a({\theta}_g, g) = \frac{1}{c_g}M{\theta_g}$.
\end{lemma}
\begin{proof}
   Each agent solves \[a({\theta}_g) = \argmax_{a' \in \mathbb{R}^d} {\theta}_g^T[x_g + Ma'] - \frac{c_g}{2}||a'||_2^2.\]
   We can check the first order optimality condition to see that $a({\theta}_g)$ must solve
   \[M^T \theta_g -  c_g a(\theta_g) = 0\]
   By assumption $M$ is diagonal and thus symmetric, so $M^T = M$ and $\frac{1}{c_g} M {\theta}_g =  a({\theta}_g)$.
\end{proof}
\begin{lemma}
    Any policy $\theta = (\theta_A, \theta_D)$ that satisfies $(\theta_g^T X'_g, \beta^T X'_g) \ind {G}$ also satisfies equality of treatment and equality of outcomes. 
\end{lemma}
\begin{proof}
     This can be seen by applying, the tower property of conditional expectation, conditioning on $X$, then use the assumption that $(\theta_g^T X'_g, \beta^T X'_g) \ind {G}.$ For example, to see that equality of responses holds note that we have:
    $$\mathbb{E}[Y'|G=A] = \mathbb{E}_X' \mathbb{E}[Y'|G=A, X'] = \mathbb{E}_X'[\sigma(\beta^TX')|G=A] $$
    $$= \mathbb{E}_X'[\sigma(\beta^T X')|G=D] = \mathbb{E}[Y'|G=D]$$
    Equality of treatment will follow from the exact same argument applied to the quantities $\hat{Y}'$ and $Y' \hat{Y}'$ (see also example \ref{ex: moment-binary-class}).
     
\end{proof}
We now provide the explicit structure of the stratified manifolds in theorem \ref{thm: ex-ante discrimination}, and corresponding proofs. For notational convenience, let $\text{Aff}(k)$ (resp. $\mathbb{S}(k)$) be the set of k-dimensional affine subspaces (resp. k-dimensional hyperspheres) of a context-dependent ambient space, and $\mathbb{O}(k)$ the group of $k \times k$ orthogonal matrices.
\begin{theorem}[Theorem \ref{thm: ex-ante discrimination} ex-ante skill discrimination]
   Consider the learning setting of Example \ref{ex: causal strategic classification}. Suppose $c_A = c_D = 1$, the vectors $\{(\mu_{A,m}, -\mu_{D,m}), (\beta_m, -\beta_m)\}$ are not co-linear and the vectors $\{(\mu_{A,u}, -\mu_{D,u}), (\beta_m, -\beta_m)\}$ are not co-linear. Then there exist sets $\mathcal{Z}_m \subset R^{2d_m}$ and $\cZ_u \in R^{2d_u}$ of the form:
   \[\mathcal{Z}_m = \cup_{U \in \mathcal{U}_m} {Z}_m(U); \: Z_m(U) \in \text{Aff}(d_m-2) \]
\[\mathcal{U}_m = \{U \in \mathbb{O}(d_m):
        \begin{pmatrix}
\beta_m^T ,-\beta_m^T\\
\mu_{A,m}^T, -\mu_{D,m}^T
\end{pmatrix} \not\perp \text{Null}[I_{d_m},-U^T] \}\]
\[\mathcal{Z}_u = \cup_{U \in \mathcal{U}_u} {Z}_u(U);  Z_u(U) \in \text{Aff}(d_u-2)\]
\[\mathcal{U}_u = \{U \in \mathbb{O}(d_u):
        \begin{pmatrix}
\beta_u^T ,-\beta_u^T\\
-\mu_{A,u}^T, \mu_{D,u}^T
\end{pmatrix} \not\perp \text{Null}[I_{d_u},-U^T] \}\]
such that for any learner decision $\theta = (\theta_A, \theta_D)$ which satisfies $(\theta_{A,m}, \theta_{D,m}) \in \cZ_m$ and $(\theta_{A,u}, \theta_{D,u}) \in \cZ_u$ also satisfies equality of treatment.
\end{theorem}
\begin{proof}
We have that any policy $\theta = (\theta_A, \theta_D)$ that satisfies or $(\theta_g^T X'_g, \beta^T X'_g) \ind {G}$ also satisfies equality of treatment. By Lemma \ref{lem: strat-resp closed form} $(\theta_g^T X'_g, \beta^T X'_g)$ are generated in the manner $\beta^T X'_g = \beta^T(X_g + M{\theta_g})$ and $\theta_g^T X'_g = {\theta}_g^T (X_g +M{\theta_g})$. 
    Under the normality assumption in \ref{ass: labormarket-simplify}, the ex-post joint distribution of the (pre-discretized) responses and predictions conditioned on sensitive trait is
\[(\theta_g^T X'_g, \beta^TX'_g) \sim \mathcal{N}(\tilde{\mu}_g, \tilde{\Sigma}_g)\]
\[\tilde{\mu}_g = (\theta^{T}_g (\mu_g + M\theta_g), \beta^{T}(\mu_g+M\theta_g))\]
\[\tilde{\Sigma}_g = \begin{pmatrix}
||\theta_g||^2 & \theta^{T}_g\beta \\
\theta^{T}_g\beta  & ||\beta||^2 
\end{pmatrix}\]

As mentioned, equality of responses plus equality of treatment will be upheld by any policy that satisfies $(\theta_A^T X'_A, \beta^TX'_A) \deq (\theta_D^T X'_D, \beta^TX'_D)$. Thus, the equality of treatment constraint set can be studied by setting $\tilde{\mu}_A = \tilde{\mu}_D$ and $\tilde{\Sigma}_A = \tilde{\Sigma}_D$, the resulting constraints are:
\begin{eqnarray}
    \theta_A^T M \theta_A &=& \theta_D^T M \theta_D\nonumber\\
    \theta_A^T M \beta + \beta^T \mu_A &=& \theta_D^T M \beta + \beta^T \mu_D\nonumber\\
    \theta_A^T \mu_A &=& \theta_D^T \mu_D \\
    ||\theta_A||^2 &=& ||\theta_D||^2\nonumber\\
    \theta_A^T \beta &=& \theta_D^T \beta   \nonumber
    \label{eq: combined-constraints}
\end{eqnarray}
The next observation is that we can decompose this feasibility requirement into two, one which pertains to parameters that correspond to ``manipulable" features, and another which pertains to ``non-manipulable" features. For a general vector $v$ let $v_{m(i)} = {v_i 1_{\{M_i = 1\}}} $ and $v_{u(i)} = {v_i 1_{\{M_i = 0\}}} $.
The constraint \aptLtoX{C.2}{\ref{eq: combined-constraints}} will be satisfied by any $(\theta_A, \theta_D)$ that satisfies both the following:
\begin{eqnarray}
    ||\theta_{m, A}||^2 &=& ||\theta_{m,D}||^2\nonumber\\
    \theta_{m, A}^T \beta_m - \theta_{m, D}^T \beta_m &=& \beta^T \mu_D - \beta^T \mu_{A} \\
    \theta_{m, A}^T \mu_{m,A} &=& \theta_{m, D}^T \mu_{m, D}\nonumber
\label{eq: m-constraints}
\end{eqnarray}
\begin{eqnarray}
    ||\theta_{u, A}||^2 &=& ||\theta_{u,D}||^2\nonumber\\
    \theta_{u, A}^T \beta_u - \theta_{u, D}^T \beta_u &=& -\beta^T \mu_D + \beta^T \mu_{A} \\
    \theta_{u, A}^T \mu_{u,A} &=& \theta_{u, D}^T \mu_{u, D}\nonumber
\label{eq: u-constraints}\end{eqnarray}
This can be seen by making note of the following identities: $||\theta||_2 = ||\theta_m||^2 + ||\theta_u||^2$, $\theta^Tv = \theta_m^Tv_m + \theta_u^Tv_u$, $\theta^TM\theta = ||\theta_m||^2$, and $\theta^TMv = \theta_m^Tv_m$. The forms of constraints \aptLtoX{C.3}{\ref{eq: m-constraints}}, \aptLtoX{C.4}{\ref{eq: u-constraints}} are nearly identical so we only provide an analysis of the ``manipulable" constraints (constraint \aptLtoX{C.3}{\ref{eq: m-constraints}}). On-wards let $2d_m$ be the dimension of the parameter space $(\theta_{m, A}, \theta_{m, D})$ 

We begin with the quadratic constraint $||\theta_{m, A}||^2 = ||\theta_{m,D}||^2$. The key observation is that this is satisfied if and only if there exists $U \in \mathbb{O}(d_m)$ such that $\theta_{m, A} = U \theta_{m, D}$. Thus, for a fixed $U \in \mathbb{O}(d_m)$ we can write the constraint set as
\[\mathcal{Z}(U): (\theta_{m, A}, \theta_{m, D}) \text{ s.t.}: \left(\begin{array}{cc}
\beta_m & -\beta_m \\
\mu_{m, A}  & -\mu_{m, D}\\
I & -U^T
\end{array}\right) \left(\begin{array}{@{}c@{}}
\theta_{m, A} \\
\theta_{m, D}
\end{array}\right) = \left(\begin{array}{@{}c@{}}
b_0 \\
0\\
\:\: \: \: \:0_{d_m}
\end{array}\right).\]
Any choice of matrix $U$ is satisfactory. However, in order to exclude any $U$ that results in $\text{dim}[Z(U)] = 0$, we only select $U$ such that $(\beta_m, -\beta_m)^T \not \perp \text{Null}[I, -U^T ]$ and $(\mu_{m,A}, -\mu_{m, D})^T \not \perp \text{Null}[I, -U^T ]$ Together, the constraint set is a union of $d_m-2$ dimensional subspaces:
\[\mathcal{Z} = \cup_{U \in \mathcal{U}} \mathcal{Z}(U)\]
\[\mathcal{U} = \{U \in \mathbb{O}(d_m):
        \begin{pmatrix}
\beta_m ,-\beta_m\\
\mu_{A,m}, -\mu_{D,m}
\end{pmatrix} \not\perp \text{Null}[I,-U^T] \}\]
\end{proof}
\begin{proof}[Proof of corollary \ref{corr:ex-ante discrim Lambda} (ex-ante skill discrimination)]
    Note that now the joint distribution of $(\theta_g^T X'_g, \beta^TS'_g)$ satisfies
    \[(\theta_g^T X'_g, \beta^TS'_g) \sim \mathcal{N}(\tilde{\mu}_g, \tilde{\Sigma}_g)\]
\[\tilde{\mu}_g = (\theta^{T}_g \Lambda (\mu_g + M\Lambda^T\theta_g), \beta^{T}(\mu_g+M\Lambda^T\theta_g))\]
\[\tilde{\Sigma}_g = \begin{pmatrix}
\theta_g^T \Lambda \Lambda^T \theta_g + ||\theta_g||^2 & \beta^T\Lambda^T\theta^{}_g \\
\beta^T\Lambda^T\theta^{}_g  & ||\beta||^2 + 1 
\end{pmatrix}\]
The constraint $||\theta_A||_2^2 = ||\theta_D||_2^2$ is satisfied for any $U \in \mathbb{O}(q)$ and pair $\theta$ such that $\theta_A = U \theta_D$, thus for every $U$ there is a $q$-dimensional plane that satisfies the constraint  $||\theta_A||_2^2 = ||\theta_D||_2^2$. We begin with the initial stratified manifold $\cZ' = \cup_{U \in \mathbb{O}_q} \text{Null}[I -U^T]$. By the above, any $\theta \in \cZ'$ satisfies $||\theta_A||^2 = ||\theta_D||^2$.
Let $V'_U \in \cZ$ satisfy $ \text{span}\{\begin{pmatrix}
        \Lambda^ T & 0\\
        0 & \Lambda^T
    \end{pmatrix}v; v\in V'_U\} \} \cong \mathbb{R}^{2d}\}$. Under the transformation $T: \theta \rightarrow \tilde{\theta} \in \mathbb{R}^{2d}; \: \tilde{\theta} = (\Lambda^T \theta_A, \Lambda^T \theta_D)$, the remaining constraints (in terms of $\tilde{\theta}$) are identical to those in the proof of \ref{thm: ex-ante discrimination}. Thus, by theorem $\ref{thm: ex-ante discrimination}$ there is a manifold $\tilde{\cM} \subset \mathbb{R}^{2d}$ such that any $\tilde{\theta} \in \tilde{\cM}$ satisfies all remaining constraints necessary for equality of treatment. Thus, any $\theta \in \cM(U) \triangleq T^{-1}(\tilde{M}) \cap V'_U$ satisfies equality of treatment, and by the full rank property of $\Lambda$ restricted to $V'_U$, $\cM(U)$ is a manifold of dimension at least $\cO(d)$.
\end{proof}
\begin{theorem}[Theorem \ref{thm: ex-ante discrimination} cost discrimination] 
\label{thm: response discrimination}
Consider the causal strategic classification setting (example \ref{ex: causal strategic classification}). Suppose $\mu_A = \mu_D = 0$ and $c_A \neq c_D$. There exist sets $\cZ_{m, A}$, $\cZ_{m, D}$, $\cZ_{u, A}$, $\cZ_{u, D}$ of the form:
    \[\cZ_{m, A} = \cup_{(k_1, k_2) \in \cK} \: \cS^{A,m}_{(k_1, k_2)}; \cS^{A,m}_{(k_1, k_2)} \in \mathbb{S}(d_m-2) \]
     \[\cZ_{D, m} = \cup_{(k_1, k_2) \in \cK} \: \cS^{D,m}_{(k_1, k_2)}; \cS^{U,m}_{(k_1, k_2)} \in \mathbb{S}(d_m-2) \]
     \[\cZ_{A, u} = \cup_{(k_1, k_2) \in \cK} \: \cS^{A,u}_{(k_1, k_2)}; \cS^{A,u}_{(k_1, k_2)} \in \mathbb{S}(d_u-2) \]
     \[\cZ_{D, u} = \cup_{(k_1, k_2) \in \cK} \: \cS^{D,u}_{(k_1, k_2)}; \cS^{D,u}_{(k_1, k_2)} \in \mathbb{S}(d_u-2) \]
     \[\cK = \{(k_1, k_2) \in \mathbb{R}^+\times \mathbb{R}: k_1 > \frac{\sqrt{\frac{c_A}{c_D}}|k_2|}{\text{min}(||\beta_u||, ||\beta_m||)}\} \]
     such that for any policy $\theta = (\theta_{A, m}, \theta_{A, u}, \theta_{D, m}, \theta_{D, u})$ which satisfies $\theta \in \cZ_{m, A} \times \cZ_{m, D} \times \cZ_{u, A} \times \cZ_{u, D}$ also satisfies \emph{ex-post} equality.
\end{theorem}
\begin{proof}
 By lemma \ref{lem: strat-resp closed form} $(\theta_g^T X'_g, \beta^T X'_g)$ are generated in the manner $\beta^T X'_g = \beta^T(X_g + \frac{1}{c_g}M{\theta_g})$ and $\theta_g^T X'_g = {\theta}_g^T (X_g +\frac{1}{c_g}M{\theta_g}+\epsilon)$. 

  Under the normality assumption in \ref{ass: labormarket-simplify}, the ex-post joint distribution of the (pre-discretized) outcomes and predictions conditioned on sensitive trait is
\[(\theta_g^T X'_g, \beta^TX'_g) \sim \mathcal{N}(\tilde{\mu}_g, \tilde{\Sigma}_g)\]
\[\tilde{\mu}_g = (\frac{1}{c_g}\theta^{T}_gM\theta_g, \frac{1}{c_g}\beta^{T}M\theta_g)\]
\[\tilde{\Sigma}_g = \begin{pmatrix}
||\theta_g||^2 & \theta^{T}_g\beta \\
\theta^{T}_g\beta  & ||\beta||^2  
\end{pmatrix}\]
     By the above lemma equality of treatments will be satisfied by any policy that satisfies $(\theta_A^T X'_A, \beta^TX'_A) \deq (\theta_D^T X'_D, \beta^TX'_D)$. Thus, using a nearly identical argument as the proof for the first part of theorem \ref{thm: response discrimination} the equality of treatments constraint set can be studied by setting $\tilde{\mu}_A = \tilde{\mu}_D$ and $\tilde{\Sigma}_A = \tilde{\Sigma}_D$, the resulting constraints are:
\begin{eqnarray}
    \frac{1}{c_A}\theta_A^T M \theta_A &=& \frac{1}{c_D}\theta_D^T M \theta_D\nonumber\\
    \frac{1}{c_A}\theta_A^T M \beta &=& \frac{1}{c_D}\theta_D^T M \beta\\
    ||\theta_A||^2 &=& ||\theta_D||^2\nonumber\\
    \theta_A^T \beta &=& \theta_D^T \beta   \nonumber
    \label{eq: combined-constraints-2}
\end{eqnarray}
Again, we can decompose this feasibility requirement into four constraints, each which pertains to parameters that correspond to ``manipulable"/``non-manipulable" and a sensitive trait value .
The constraint \aptLtoX{C.5}{\ref{eq: combined-constraints-2}} will be satisfied by any $(\theta_A, \theta_D)$ that satisfies all four the following for any pair $(k_1, k_2)$:
\begin{eqnarray}
    ||\theta_{m, A}||^2 &=& \frac{c_A}{c_D} k_1^2\nonumber\\
    \theta_{m, A}^T \beta_m &=& \frac{c_A}{c_D}k_2
    \label{eq: A-m-constraints-2}
\end{eqnarray}
\begin{eqnarray}
    ||\theta_{m,D}||^2 &=& k_1^2\nonumber\\
     \theta_{m, D}^T \beta_m = k_2
     \label{eq: D-m-constraints-2}
\end{eqnarray}
\begin{eqnarray}
    ||\theta_{u, A}||^2 &=& k_1^2\nonumber\\
    \theta_{u, A}^T \beta_u &=& k_2 
    \label{eq: A-u-constraints-2}
\end{eqnarray}
\begin{eqnarray}
    ||\theta_{u,D}||^2 &=& \frac{c_A}{c_D} k_1^2\nonumber\\
    \theta_{u, D}^T \beta_u &=& \frac{c_A}{c_D} k_2
    \label{eq: D-u-constraints-2}
\end{eqnarray}
We have again used each of the following identities $||\theta||_2 = ||\theta_m||^2 + ||\theta_u||^2$, $\theta^Tv = \theta_m^Tv_m + \theta_u^Tv_u$, $\theta^TM\theta = ||\theta_m||^2$, and $\theta^TMv = \theta_m^Tv_m$. Consider the first constraint set 
\aptLtoX{C.5}{\ref{eq: A-m-constraints-2}}; the constraint $||\theta_{m, A}||^2 = \frac{c_A}{c_D} k_1^2$ is a hypersphere (of dimension $d_{m,A} -1$) with radius $ \frac{c_A}{c_D} k_1^2$. The constraint  $\theta_{m, A}^T \beta_m =  \frac{c_A}{c_D}k_2 \iff \theta_{m, A}^T \frac{\beta_m}{||\beta_m||} = \frac{c_A}{c_D}\frac{k_2}{||\beta_m||}$ is a hyperplane. It is easy to see that the intersection between these two geometric objects is either empty, (if $|\frac{c_A}{c_D}\frac{k_2}{||\beta_m||}|> \sqrt{\frac{c_A}{c_D} k_1^2}$ ) or is a hypersphere of dimension $d_{m,A}-2$ (if $|\frac{c_A}{c_D}\frac{k_2}{||\beta_m||}|< \sqrt{\frac{c_A}{c_D} k_1^2}$). Note that an identical argument can be applied to each of the other constraint sets, \ref{eq: D-m-constraints-2}, \ref{eq: A-u-constraints-2}, \ref{eq: D-u-constraints-2}. Each of these sets will be non-empy so long as $|k_1| > \frac{\sqrt{\frac{c_A}{c_D}}k_2}{\text{min}(||\beta_u||, ||\beta_m||)}$.
\end{proof}
\begin{proof}[Proof of corollary \ref{corr:ex-ante discrim Lambda} (cost discrimination)]
    Note that now the joint distribution of $(\theta_g^T X'_g, \beta^TS'_g)$ satisfies
    \[(\theta_g^T X'_g, \beta^TS'_g) \sim \mathcal{N}(\tilde{\mu}_g, \tilde{\Sigma}_g)\]
\[\tilde{\mu}_g = (\theta^{T}_g \Lambda (\frac{1}{c_g} M\Lambda^T\theta_g), \beta^{T}(\frac{1}{c_g}M\Lambda^T\theta_g))\]
\[\tilde{\Sigma}_g = \begin{pmatrix}
\theta_g^T \Lambda \Lambda^T \theta_g + ||\theta_g||^2 & \beta^T\Lambda^T\theta^{}_g \\
\beta^T\Lambda^T\theta^{}_g & ||\beta||^2 + 1 
\end{pmatrix}\]
The constraint $||\theta_A||_2^2 = ||\theta_D||_2^2$ is satisfied for any $U \in \mathbb{O}(q)$ and pair $\theta$ such that $\theta_A = U \theta_D$, thus for every $U$ there is a $q$-dimensional plane that satisfies the constraint  $||\theta_A||_2^2 = ||\theta_D||_2^2$. We begin with the initial stratified manifold $\cZ' = \cup_{U \in \mathbb{O}_q} \text{Null}[I -U^T]$, using the above, any $\theta \in \cZ'$ satisfies $||\theta_A||^2 = ||\theta_D||^2$.
Let $V'_U \in \cZ$ satisfy $ \text{span}\{\begin{pmatrix}
        \Lambda^ T & 0\\
        0 & \Lambda^T
    \end{pmatrix}v; v\in V'_U\} \} \cong \mathbb{R}^{2d}\}$. Under the transformation $T: \theta \rightarrow \tilde{\theta} \in \mathbb{R}^{2d}; \: \tilde{\theta} = (\Lambda^T \theta_A, \Lambda^T \theta_D)$, the remaining constraints (in terms of $\tilde{\theta}$) are identical to those in the proof of \ref{thm: response discrimination}. Thus, by theorem $\ref{thm: response discrimination}$ there is a manifold $\tilde{\cM} \subset \mathbb{R}^{2d}$ such that any $\tilde{\theta} \in \tilde{\cM}$ satisfies all remaining constraints necessary for equality of treatment. Thus, any $\theta \in \cM(U) \triangleq T^{-1}(\tilde{M}) \cap V'_U$ satisfies equality of treatment, and by the full rank property of $\Lambda$ restricted to $V'_U$, $\cM(U)$ is a manifold of dimension at least $\cO(d)$.
\end{proof}
\subsection{Proof of impossibility results}
\begin{proof}[Proof of theorem \ref{thm: informal}:]
Under the unbiasedness assumption on test signals, the optimal hiring policy for the firm will be of the form $f(x, g) = \mathbf{1}\{x\geq \theta_g\}$. Consider a firm that uses a group blind threshold policies of the form $f(x, g) = \mathbf{1}\{x\geq \theta\}$ to hire workers, where $\theta\in\reals$ is the threshold value for both groups. The workers' expected utility for increasing skills from $s$ to $s^*$ is
\begin{equation}\textstyle
u_w(f,s,s^*, g) = w\bar{F}(\theta \mid s^*) - c(s,s^*),
\label{eq:workers-utility}
\end{equation}
where $\bar{F}(\theta\mid s^*) \triangleq \Pr\{X > \theta\mid S=s^*\}$ is the survival function of the skill level assessment.

In the labor market example, the worker (in group $g$) response 
\[\textstyle
s'(s, f g)\triangleq\argmax_{s^*}u_w(f,s,s^*, g)
\]
is generally non-decreasing. For example, suppose $c_A = c_D = c$ and $c(s,s^*) = \frac{c}{2}[s^*-s]_+^2$, where  $[\,\cdot\,]_+\triangleq\max\{0,\cdot\}$ is the ReLU function, then the derivative of the worker response is
\[
\frac{\partial s'(s, f, g)}{\partial_s} = -\frac{\partial_s\partial_{s^*}c(y,s^*)|_{s^* = s'(s, f, g)}}{\partial_{s^*}^2u_w(f,y,s^*)|_{s^* = s'(s,f,g)}} = \frac{c - w\partial_y^2\bar{F}(t\mid s'(s,f,g))}{c},
\]
so $\frac{\partial s'(s, f, g)}{\partial_s} > 0$ as long as $c$ is large enough. In general, We are unconcerned with settings in which $ s'(s, f, g)$ is not non-decreasing in $s$ this would be both unintuitive and unrealistic. Under the assumption that $S|A$ stochastically dominates $S|D$, the non-decreasingness of $ s'(s, f, g)$ in $s$ implies $S'|A$ is stochastically dominates $S'|D$, which precludes equal \emph{ex post} responses (in particular $\mathbb{E}[Y'|A] > \mathbb{E}[Y'|D]$).

On the other hand, suppose that $S|A \deq S|D$ but $c_A < c_D$, keeping the assumption that policies are group blind. Note that
\[\frac{\partial}{\partial_c} s'(s, f, c) = \frac{w \partial_{s^*}^2 \bar{F}(t|s^*)|_{s^* = s'(s,f,c)} - c}{[s'(s, f, c) - s]_+}\]
This implies that if $c_A, c_D$ are large enough with $c_A < c_D$ two workers from each group with $s_A = s_D$ will have $s'(s, f, c_A) > s'(s, f, c_D)$ implying $S'|A$ stochastically dominates $S'|D$, which precludes equal \emph{ex post} responses. 

Finally, the fact that only group blind policies will satisfy equality of outcomes with respect to $m_{\text{FPR}}(f, g)$ and $m_{\text{FNR}}(f, g)$ follows immediately from the market assumption that $X'$ is independent of $G$ given $Y'$.
\end{proof}
\section{Appendix A Proofs}
\subsection{Proof of Proposition \ref{prop: linearity} }
\begin{proof}
We prove the statement on $\mu(Q)$, the proof for $\text{EPR}(Q)$ is identical.
We have: \[\mu_{ij}(Q) = \mathbb{E}_{(X', Y'), Q, \cG}[h_i(Q(X'), Y')|\mathcal{G} = g_j]\]
By the law of total (conditional) expectation this is equivalent to
\[\mu_{ij}(Q) = \sum_{f \in \cF} P(Q = f) \mathbb{E}_{(X', Y'), Q}[h_i(Q(X'), Y')|{G}\] \[= g_j, Q = f]P(Q=f|{G}=g_j)\]
By assumption $P(Q=f|{G}=g_j) = P(Q=f)$. Note also that $ \mathbb{E}_{(X', Y') \sim \cD(Q), Q}[h_i(Q(X'), Y')|{G} = g_j, Q = f] = \mathbb{E}_{(X', Y') \sim \cD(f)}[h_i(f(X'), Y')|\mathcal{G} = g_j] = \mu_{ij}(f)$. Thus
\[\mu_{ij}(Q) = \sum_{f \in \cF} P(Q=f) \mu_{ij}(f) =  \sum_{f \in \cF} P(Q=f) \mu_{ij}(f) = \sum_{f \in \cF} Q(f) \mu_{ij}(f)\]
\end{proof}
\subsection{Proof of Generalization Error}
\begin{proposition}[\citet{agarwal2018Reductions} Lemma 3]
\label{prop: fairness_optimization_error}
    Suppose that the constraint $M\hat{\mu}(Q) \leq \hat{c}$ is feasible. Then if $\hat{Q}$ is a $\epsilon$-saddle point of
    equation \ref{eq: sample dual} the following holds:
    \[||M\hat{\mu}(\hat{Q})-\hat{c}||_{\infty} \leq \frac{1+2\epsilon}{B}\]
\end{proposition}

\begin{proposition}[\citet{agarwal2018Reductions} Lemma 2]
\label{prop: risk_optimization_error}
    If $\hat{Q}$ is an $\epsilon$-saddle point, then for all $Q$ such that $M\hat{\mu}(Q) \leq c$ the distribution $\hat{Q}$ satisfies 
    \[\widehat{\text{EPR}}(\hat{Q}) \leq \widehat{\text{EPR}}({Q}) + 2\epsilon.\]
\end{proposition}
The technical tool we use to study the generalization properties of algorithm \ref{alg: ex-post-reduction} is the Rademacher complexity, which we now define.
\begin{definition}
    Let $\cF$ be a class of functions $f: \cX \rightarrow [0,1]$ and $\epsilon_i$ be i.i.d. Rademacher random variables. the Rademacher complexity of $\cF$ is defined as
    \[\cR_n(\cF) \triangleq \sup_{x_1, \dots, x_n \in \cX} \mathbb{E}_{\epsilon} \sup_{f\in\cF} |\frac{1}{n}\sum_{i=1}^n \epsilon_i f(x_i)|\]
\end{definition}
The primary obstacle to proving a generalization bound will be to establish bounds on the Rademacher complexity of the function classes $\mu_i(\theta), \hat{\text{EPR}}(\theta)$. Beyond this the analysis is a standard application of the arguments in \citep{agarwal2018Reductions}.
\begin{lemma}
\label{lem: Error Analysis}
    Let $\cX \triangleq\{x\in \mathbb{R}^d; ||x|| \leq C\}$, let $\cH_{\theta} = \{\theta^Tx; x\in \cX, \theta \in \Theta\}$.   $$\hat{\mu}_1(\theta) = \sum_{i=1}^n\sigma(\theta^T(x_i+Ma(\theta))),$$
    $$\hat{\mu}_2 = \sum_{i=1}^n\sigma(\beta^T(x_i+Ma(\theta))),$$
    $$\hat{\mu}_3 = \sum_{i=1}^n \sigma(\theta^T(x_i+Ma(\theta)))\sigma(\beta^T(x_i+Ma(\theta))).$$
    Then the following hold:
    \begin{enumerate}
        \item With probability at least $1-\delta$ for all $\theta \in \Theta$, 
        \[|\hat{\mu_1}(\theta) - \mu_1(\theta)| \leq 2R_n(\cH_{\theta}) + 2\sup_{\theta \in \Theta}|\theta^TMa(\theta)|/(\sqrt{n})+\sqrt{\frac{ln(2/\delta)}{2n}} \sim \tilde{\cO}(n^{-1/2})\]
        \item With probability at least $1-\delta$ for all $\theta \in \Theta$ 
        \[|\hat{\mu_2}(\theta) - \mu_2(\theta)| \leq 2||\beta||C/\sqrt{n}+2\sup_{\theta \in \Theta}|\beta^TMa(\theta)|/\sqrt{n} + \sqrt{\frac{ln(2/\delta)}{2n}} \sim \tilde{\cO}(n^{-1/2})\]
        \item With probability at least $1-\delta$ for all $\theta \in \Theta$
        \[|\hat{\mu_3}(\theta) - \mu_3(\theta)|\leq\] \[2\sqrt{2}||\beta||C/\sqrt{n}+2\sqrt{2}\sup_{\theta \in \Theta}|\beta^TMa(\theta)|/\sqrt{n} +  2\sqrt{2}R_n(\cH_{\theta})\] \[+ 2\sqrt{2}\sup_{\theta \in \Theta}|\theta^TMa(\theta)|/(\sqrt{n})+\sqrt{\frac{ln(2/\delta)}{2n}}\sim \tilde{\cO}(n^{-1/2})\]
    \end{enumerate}
\end{lemma}
\begin{proof}
\begin{enumerate}
    \item Let $\cF_{\theta}$ be the class of functions $f_{\theta}: s\rightarrow \sigma(\theta^T(x+Ma(\theta))) \in[0,1]$. By a standard concentration inequality, with probability at least $1-\delta$ for all $\theta \in \Theta$,
    \[ |\hat{\mu_1}(\theta) - \mu_1(\theta)| \leq 2R_n(\cF_{\theta}) + \sqrt{\frac{ln(2/\delta)}{2n}}\]
    Thus it remains to find an appropriate bound for $R_n(\cF_{\theta})$.
    Note that 
    \[R_n(\cF_{\theta}) = \sup_{x_1, \ldots, x_n \in \cX} \mathbb{E}_{\epsilon_i}[\sup_{\theta \in \Theta}|\frac{1}{n}\sum_{i=1}^n \epsilon_i \sigma(\theta^T(x_i+Ma(\theta)))|]\] \[\stackrel{\text{Talegrands}}{\leq}\sup_{x_1, \ldots, x_n \in \cX} \mathbb{E}_{\epsilon_i}[\sup_{\theta \in \Theta}|\frac{1}{n}\sum_{i=1}^n \epsilon_i \theta^Tx_i+\epsilon_i \theta^TMa(\theta)|]\]
    \[\leq \sup_{x_1, \ldots, x_n \in \cX} \mathbb{E}_{\epsilon_i}[\sup_{\theta \in \Theta}|\frac{1}{n}\sum_{i=1}^n \epsilon_i \theta^Tx_i|]+\mathbb{E}_{\epsilon_i}[\sup_{\theta \in \Theta}|\frac{1}{n}\epsilon_i \theta^TMa(\theta)|]\]
    \[\leq R_n(\cH_{\theta})+\sup_{\theta \in \Theta}|\theta^TMa(\theta)|/\sqrt{n}\]
    \item By an identical argument to the case of $\mu_1$ we have that with probability at least $1-\delta$ for all $\theta \in \Theta$ we have
    \[|\hat{\mu_2}(\theta) - \mu_2(\theta)| \leq \sup_{x_1, \ldots, x_n \in \cX} \mathbb{E}_{\epsilon}|\frac{1}{n}\sum_{i=1}^n \epsilon_i \beta^T x_i| + \sup_{\theta \in \Theta}|\beta^TMa(\theta)|/\sqrt{n}\]
    The first term is trivially bounded above by the empirical Rademacher complexity of the function class $\{w^Tx; x\in \cX, ||w||_2\leq ||\beta||_2\}$ which is bounded above by $||\beta||C/\sqrt{n}$
    \item Let $\cB_{\theta}$ be the function class $b_{\theta}: x\rightarrow \sigma(\beta^T(x+Ma(\theta))) \in[0,1]$, and let $\Pi_{\theta}$ be the function class $\pi_{\theta}: x\rightarrow \sigma(\beta^T(x+Ma(\theta)))\sigma(\theta^T(x+Ma(\theta))) \in[0,1]$. Note that this second function class can be thought of as the composition of the (1-Lipschitz on $[0,1] \times [0,1]$) function $\psi(x,y) = xy$ and the vector valued function $v(x) = (\sigma(\beta^T(x+Ma(\theta))), \sigma(\theta^T(x+Ma(\theta))))$. Thus by corollary 4 of \citep{maurer2016vectorcontraction}  we have 
    \[\cR_n(\Pi_{\theta}) \leq \sqrt{2}[\cR_n(\cB_{\theta})+\cR_n(\cF_{\theta})]\]
    From here we simply plug in the upper bounds for $\cR_n(\cB_{\theta})$ and $\cR_n(\cF_{\theta})$ attained in part 1 and 2 and the standard concentration inequality used throughout.
\end{enumerate} 
The fact that each quantity (1,2,3) is $\sim \tilde{\cO}(n^{-1/2})$ follows from the well-known result that $\cR_n(\cH_{\theta}) \sim \cO(n^{-1/2})$ if $x$ and $\theta$ are bounded. 
\end{proof}
\begin{proof}[\textit{Proof of theorem} \ref{thm: error analysis}]
\phantom{proof}

Note that $||M\mu(\hat{Q})||_{\infty} \leq ||M(\hat{\mu}(\hat{Q}) - \mu(\hat{Q}))||_{\infty} + ||M\hat{\mu}(\hat{Q})||_{\infty}$. By proposition \ref{prop: fairness_optimization_error} (and choice of $\nu$) it holds that $||M\hat{\mu}(\hat{Q})||_{\infty} \leq \frac{1+2\epsilon}{B}$. By the form of $M$, $||M(\hat{\mu}(\hat{Q}) - \mu(\hat{Q}))||_{\infty} \leq 2 ||\hat{\mu}(\hat{Q}) - \mu(\hat{Q})||_{\infty}$. Then by lemma \ref{lem: Error Analysis} and a union bound, with probability at least $1-6\delta$, it holds that $||\hat{\mu}(\hat{Q}) - \mu(\hat{Q})||_{\infty} \leq \cO(\text{min}(n_A, n_D)^{-1/2})+8\sqrt{\frac{ln(2/\delta)}{2\text{min}(n_A, n_D)}}$

Additionally, by our choice of $\nu$, $||M(\hat{\mu}(Q^*))|| \leq \hat{c}$, so by proposition \ref{prop: risk_optimization_error} $\widehat{\text{EPR}}(\hat{Q}) \leq \widehat{\text{EPR}}({Q^*}) + 2\epsilon$. By the argument of part 3 of lemma $\ref{lem: Error Analysis}$ (and the additive nature of Rademacher complexity), with probability at least $1-\delta$
\[|\widehat{\text{EPR}}(\hat{Q}) -{\text{EPR}}(\hat{Q})| \leq \cO(\text{min}(n_A, n_D)^{-1/2}) + \sqrt{\frac{ln(2/\delta)}{2\text{min}(n_A, n_D)}}\]
\[|\widehat{\text{EPR}}({Q^*}) -{\text{EPR}}({Q^*})| \leq \cO(\text{min}(n_A, n_D)^{-1/2}) + \sqrt{\frac{ln(2/\delta)}{2\text{min}(n_A, n_D)}}\]
Thus with probability at least $1-\delta$ it holds that ${\text{EPR}}(\hat{Q}) \leq {\text{EPR}}({Q^*}) + 2\epsilon + \cO(\text{min}(n_A, n_D)^{-1/2}) + 2\sqrt{\frac{ln(2/\delta)}{2\text{min}(n_A, n_D)}}$. A final union bound completes the proof.

\end{proof}
\section{Experiments}
\subsection{Experimental Details}
All experiments were done on Google Colab using only a CPU. The chosen parameters for data generation are $\beta = 1_{10}$, $\mu_A = 0.5*1_{10}$, $\mu_D = 0.1*1_{10}$, $c_A = \frac{4}{2}||a||_2^2$, $c_D =  \frac{10}{2}||a||_2^2$, $\Sigma=I_{10\times10}$, $\Lambda=I_{10\times10}$ with \emph{ex-ante} skills generated from $\cN(\mu_g, \Sigma)$ distributions. Each base line was also implemented using a reduction method. Step size $\eta_t$ was selected according to the convergence theory, $B$ was selected to achieve at least $10^{-5}$ fairness violation on the training set. Training and test sizes of 500 samples were used, with a random seed of 0 for the training set and a random seed of 1 for the test set.
\subsection{Additional Experiments}

\begin{figure}[H]
  \includegraphics[width=0.45\textwidth]{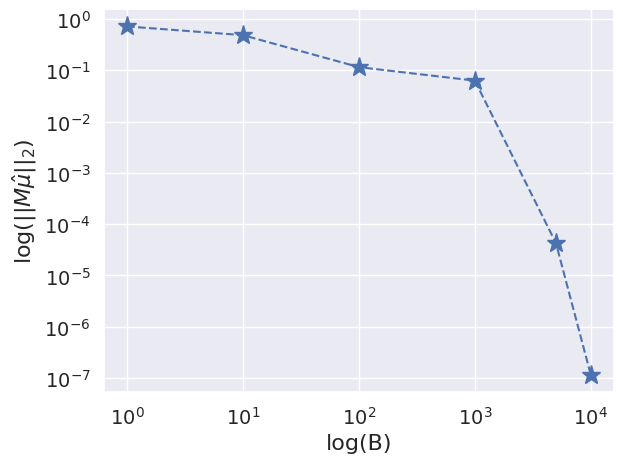}
  \caption{Equality of treatment + Equality of responses on training set}
  \label{fig:feasibility exp}
\end{figure}

\end{document}

%% file: YK-macros.tex
\def\argmin{{\arg\min}}
\def\argmax{{\arg\max}}

\def\bb{\mathbf{b}}

\def\bi{\mathbf{i}}

\def\bbE{\mathbb{E}}

\def\bbP{\mathbb{P}}

\def\bbR{\mathbb{R}}

\def\cB{\mathcal{B}}

\def\cD{\mathcal{D}}

\def\cF{\mathcal{F}}
\def\cG{\mathcal{G}}
\def\cH{\mathcal{H}}

\def\cK{\mathcal{K}}
\def\cL{\mathcal{L}}
\def\cM{\mathcal{M}}
\def\cN{\mathcal{N}}
\def\cO{\mathcal{O}}

\def\cR{\mathcal{R}}
\def\cS{\mathcal{S}}

\def\cX{\mathcal{X}}

\def\cZ{\mathcal{Z}}

\def\deq{\overset{d}{=}}

\def\Ex{{\bbE}}

\def\eg{{\it e.g.}}

\def\eps{{\epsilon}}

\def\ie{{\it i.e.}}

\def\ind{\perp \!\!\! \perp}

\def\ones{\mathbf{1}}

\def\Pr{{\bbP}}

\def\reals{{\bbR}}

%% file: YK-thm-macros.tex
\newtheorem{theorem}{Theorem}[section]
\newtheorem{assumption}[theorem]{Assumption}

\newtheorem{corollary}[theorem]{Corollary}
\newtheorem{definition}[theorem]{Definition}
\newtheorem{example}[theorem]{Example}

\newtheorem{lemma}[theorem]{Lemma}

\newtheorem{proposition}[theorem]{Proposition}